\documentclass{article}
\usepackage{graphicx} 
\usepackage{tccml_iclr2025_conference}
\usepackage{amsthm}
\usepackage{amsmath}
\usepackage{amsfonts} 
\usepackage[T1]{fontenc}
\usepackage{booktabs}
\usepackage{graphicx}
\usepackage{subcaption}
\usepackage{url}
\usepackage{soul}
\usepackage{authblk} 

\usepackage{styles}  

\newtheorem{theorem}{Theorem}[section]
\newtheorem{lemma}[theorem]{Lemma}

\newcommand{\method}[1]{\textsc{#1}}

\title{\method{Climplicit}: Climatic Implicit Embeddings for Global Ecological Tasks}

\author[1]{\textbf{Johannes Dollinger}}
\author[1]{\textbf{Damien Robert}}
\author[2]{\textbf{Elena Plekhanova}}
\author[1]{\textbf{Lukas Drees}}
\author[1]{\\\textbf{Jan Dirk Wegner}}

\affil[1]{EcoVision Lab, DM3L, University of Zurich}
\affil[2]{Swiss Federal Research Institute WSL}

\date{January 2025}

\begin{document}
\lhead{\small Published as a workshop paper at "Tackling Climate Change with Machine Learning", ICLR 2025}
\maketitle

\begin{abstract}

Deep learning on climatic data holds potential for macroecological applications. 
However, its adoption remains limited among scientists outside the deep learning community due to storage, compute, and technical expertise barriers.
To address this, we introduce Climplicit, a spatio-temporal geolocation encoder pretrained to generate implicit climatic representations anywhere on Earth. 
By bypassing the need to download raw climatic rasters and train feature extractors, our model uses $\times 3500$ less disk space and significantly reduces computational needs for downstream tasks. 
We evaluate our \method{Climplicit} embeddings on biomes classification, species distribution modeling, and plant trait regression. 
We find that single-layer probing our \method{Climplicit} embeddings consistently performs better or on par with training a model from scratch on downstream tasks and overall better than alternative geolocation encoding models.
\end{abstract}

\section{Introduction}

Ecology studies how species interact with their environment, and climate is a key factor in shaping habitats and distributions \citep{karger2017climatologies}.
Recently, deep learning has been successfully applied to the modeling of species distribution from citizen science data \citep{cole2023spatial, botella2023geolifeclef, brun2024multispecies, dollinger2024sat}.
These models rely on climatic variables, which are derived from vast amounts of raw weather data collected by global stations and aggregated into climatic reanalysis datasets.
Although deep learning is well suited to leverage such data, its adoption in ecology remains limited due to high computational, storage, and expertise requirements.

To this end, several works focus on pretraining lightweight, off-the-shelf geolocation encoders for downstream tasks \citep{rolf2021generalizable, mai2023csp, klemmer2023satclip, vivanco2024geoclip, agarwal2024general, sastry2025taxabind}. 
However, existing models often suffer from low resolution or produce inconsistent embeddings (see Appendix \ref{visuals}, \figref{fig:lvb_comp}), limiting their effectiveness.
To address these shortcomings, we introduce \method{Climplicit}, a \textit{spatio-temporal} location encoder that generates implicit climatic representations for any land location and month. 
By pretraining on \method{CHELSA} climatologies \citep{karger2017climatologies, karger2018data}, our model captures fine-grained environmental patterns with greater consistency and resolution.
Our \method{Climplicit} backbone is based on residual \method{SIREN} (\method{ReSIREN}) blocks, a novel extension of \method{SIREN} \citep{sitzmann2020implicit} with residual connections that improve scalability. 
Although previous works have explored compressing climatic data into neural networks \citep{huang2022compressing, naour2023time, liu2023scientific, xu2024hiha}, our approach specifically focuses on a high-resolution, lightweight location encoder.
We validate \method{Climplicit} embeddings on biodiversity-related tasks and provide ablation studies for our key design choices. 
Code is available at \url{https://github.com/ecovision-uzh/climplicit}.

\section{Methodology}
Let $i \in \mathbb{N}$ be an index over discretized month-locations on Earth.
We define a location encoder as a neural network \textbf{LE} that transforms the positional encoding $p_i$ of a month-coordinate triplet $(\lambda_i,\phi_i,m_i)$ with longitude $\lambda_i \in [-1,1]$, latitude $\phi_i \in [-1,1]$ and month $m_i \in \{1, ..., 12\}$ into an embedding $e_i = \text{\textbf{LE}}(p_i)$.

\subsection{\method{Direct} Positional Embedding with \method{H-SIREN}}
Our positional encoding extends with a temporal dimension the \method{Direct} embedding defined by \cite{russwurm2024locationencoding}:

\begin{equation*}
    p_i = \Big[ \: \lambda_i, \: \phi_i, \: \text{sin}(2\pi *m_i / 12), \: \text{cos}(2\pi* m_i / 12) \: \Big] \; \in \; [-1, 1]^{4} \; .
\end{equation*}

We choose the \method{Direct} embedding because of its low computational cost and sufficient representational power in conjunction with the \method{SIREN} network \citep{russwurm2024locationencoding}. 
It is also free of hyperparameters that predefine frequencies, unlike Fourier Features \citep{tancik2020fourier}, Spherical Harmonics \citep{russwurm2024locationencoding} and Sphere2Vec \citep{mai2023sphere2vec}. The sine and cosine transformations of the months ensure that the values are in $[-1, 1]$ and that January and December have close representations. 
To increase the frequencies accessible to the network, we use the \method{H-SIREN} \citep{gao2024h} activation function in the first layer (see Appendix \ref{activations}).

\begin{figure}
\centering
\includegraphics[width=.59\linewidth]{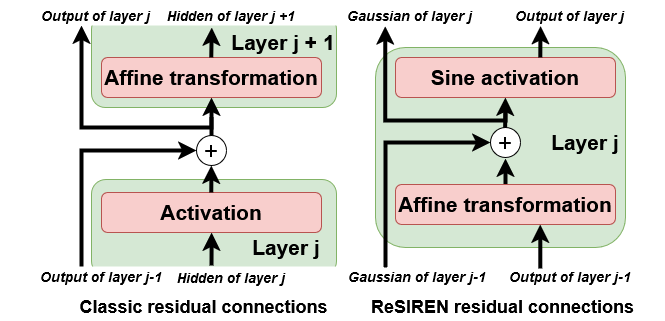}
\includegraphics[width=.39\linewidth]{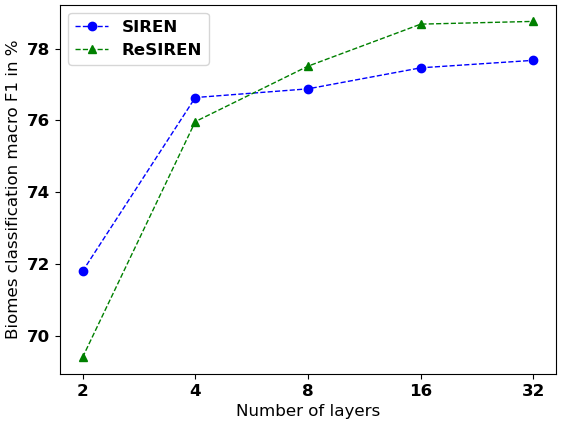}
\caption{
\textbf{Left:} 
Contrary to ``classic'' residual blocks, our \method{ReSIREN} residual connection is placed before the non-linearity rather than after.
\textbf{Right:}
Scaling comparison of \method{SIREN} and \method{ReSIREN} on the biome classification task (\secref{sec:dsts}). 
\method{ReSIREN} underperforms at small depths but scales better. 
See Appendix \ref{scaling} for further details.
}
\label{fig:ReSIREN}
\end{figure}

\subsection{Residual Connection for \method{SIREN}}
Residual connections have become a staple of neural networks since they enabled the first deep architectures \citep{he2016deep}. 
We introduce \textbf{\method{ReSIREN}}, a residual extension of \method{SIREN} \citep{sitzmann2020implicit} that improves scalability and enhances information exchange between layers \citep{he2016identity}. 
Let $L_j$, where $j \in \{1, ..., D\}$, be the $j$-th layer in a network of depth $D$, consisting of an affine transformation $F_j$ and activation function $A_j$. 
Given input $z_j$, we define the hidden pre-activation vector as $h_j = F_j(z_j)$ and the output as $z_{j+1} = L_j(z_j)$. 
We introduce residual connections between the hidden vectors:
\[
h'_j = \frac{h_j + h'_{j-1}}{2}, \quad z_{j+1} = A_j(h'_j).
\] 
This preserves the normal distribution $\mathcal{N}(0,1)$ of $h_j$ required for the stability of \method{SIREN} \citep{sitzmann2020implicit}. 
\figref{fig:ReSIREN} illustrates \method{ReSIREN} and its improved scaling properties compared to \method{SIREN}. 
See Appendix \ref{stability} and \ref{scaling} for more details.

\subsection{Climatic Pretraining}
We train a \method{ReSIREN} model on the \method{CHELSA} climate dataset \citep{karger2017climatologies}, a downscaling of ERA5 \citep{dee2011era} data using a high-resolution digital elevation model. 
We use the monthly climatology from 1981-2010, containing 11 climatic variables related to humidity, wind speed, precipitation, temperature and pressure\footnote{For more details, see section \textbf{7.2 Monthly} of \url{https://chelsa-climate.org/downloads}} that were selected in consultation with ecologists. 
A climatology refers to the per-month aggregate over 30 years.
For each of the 12 months, we use a global raster at a horizontal resolution of 30 arc sec for each of the 11 variables. 
We select only land surface pixels 
and arrive at a total number of $224$ million pixels, resulting in a file size of $12 * 11 * 0.418 \text{ GB} = 55.2\text{ GB}$ in float16. 
We pretrain our \method{ReSIREN} backbone followed by an affine projection to regress the Gaussian-normalized \method{CHELSA} values using the mean squared error loss. 
The resulting \method{Climplicit} model has a size of $16$ MB: $99.97\%$ smaller than the original \method{CHELSA} dataset is was pretrained on. 
Interestingly, the model could be further improved using neural compression \citep{neill2020overview, gomes2025lossy}. 
See Appendix \ref{training} for more details.

\section{Experiments}
We assess \textbf{\method{Climplicit}} embeddings with single-layer probing on three macroecological downstream tasks: \textbf{Biomes} classification, species distribution modeling (\textbf{SDM}) and \textbf{Plant traits} regression (see Appendix \ref{sec:dsts} for more details). 
To illustrate the storage- and compute-intensive nature of training that \method{Climplicit} aims to mitigate, we train a model from scratch (\textbf{\method{FS}}) with a location encoder (\textbf{\method{Loc}}) or a \method{CHELSA} encoder (\textbf{\method{CH}}) for each downstream task.
For \method{Loc}, we use a 16-layer \method{ReSIREN} with the same hyperparameters as \method{Climplicit}, but with $[\lambda, \phi]$ as input. For \method{CH}, we incorporate \method{CHELSA} values from March, June, September, and December, processing them with a 4-layer residual network following \cite{cole2023spatial}. 
For \method{Loc $+$ CH}, we concatenate both embeddings before the probe.
Additionally, we compare \method{Climplicit} embeddings with those from \textbf{\method{SatCLIP}}~\citep{klemmer2023satclip}, \textbf{\method{Taxabind}}~\citep{sastry2025taxabind}, \textbf{\method{SINR}}~\citep{cole2023spatial}, \textbf{\method{CSP}}~\citep{mai2023csp}, and \textbf{\method{GeoCLIP}}~\citep{vivanco2024geoclip}, which implicitly represent modalities such as satellite imagery, sound, bioclimatic variables, and species occurrence data.

\section{Results}
\begin{table}[ht]
    \caption{
    Single-layer probing on diverse geolocation representations.
    Mean and standard deviation of ten random initializations of the single-layer probe.
    \textbf{Best}.
    \underline{Second best}.
    }
    \centering
    \begin{tabular}{l  c  c  c  } 
     \toprule
     \textbf{Model} & Biomes (\% F1 $\uparrow$) & SDM (\% Acc $\uparrow$) & Plant traits (\% R² $\uparrow$)  \\ [0.5ex] 
     \midrule
     \method{FS Loc       }  & 73.9 $\pm$ 2.4    &  2.0 $\pm$ 0.4  &  42.2 $\pm$ 0.0  \\ 
     \method{FS CH        }  & 71.8 $\pm$ 1.9    &  2.5 $\pm$ 0.1  &  60.0 $\pm$ 0.3   \\ 
     \method{FS Loc $+$ CH  }  & \textbf{79.6} $\pm$ 1.7    &  2.5 $\pm$ 0.1      &  \underline{64.8} $\pm$ 0.4  \\ 
     \midrule
     \method{SatCLIP       }  & 68.3 $\pm$ 0.4       &  1.3 $\pm$ 0.1       &  61.6 $\pm$ 0.1    \\
     \method{Taxabind        }  & 59.3 $\pm$ 0.1       &  3.1 $\pm$ 0.0      &  56.9 $\pm$ 0.0    \\
     \method{SINR            }  & 63.1 $\pm$ 0.3       &  1.7 $\pm$ 0.0      &  63.5 $\pm$ 0.1    \\
     \method{CSP             }  & 58.6 $\pm$ 0.4       &  1.6 $\pm$ 0.1      &  49.7 $\pm$ 0.3    \\
     \method{GeoCLIP         }  & 62.7 $\pm$ 0.1       &  \textbf{3.5} $\pm$ 0.0      &  57.9 $\pm$ 0.1    \\
     \midrule
     \method{Climplicit} (Ours)  & \underline{78.4} $\pm$ 0.3 & \underline{3.2}   $\pm$ 0.0 &\textbf{70.0} $\pm$ 0.1\\[1ex]
     \bottomrule
    \end{tabular}
    \label{tab:main_results}
    \vspace{-0.05in}
\end{table}
We summarize our results in \tabref{tab:main_results}.

\method{Climplicit} representations surpass all \method{FS} ones, except for \method{FS Loc $+$ CH} on Biomes, which slightly outperforms \method{Climplicit} but with a larger confidence interval, highlighting the high variance of representations learned from scratch on small datasets compared to pretrained representations.
These results validate the potential of \method{Climplicit} for scenarios with minimal storage, compute, or deep learning knowledge.
Furthermore, we observe that other geolocation encoders fall short of \method{Climplicit} on all tasks, with the exception of \method{GeoCLIP} on SDM, which we attribute to the high resolution of \method{GeoCLIP} embeddings.
This highlights the relevance of high-resolution climatic representations for the downstream tasks at hand. Note in Appendix \ref{mlp_exps} Table \ref{tab:mlp_results} that using MLP probing lets \method{Climplicit} outperform the other methods on all three tasks.
Further quantitative and qualitative results can be found in Appendix \ref{non_eco_section}, \ref{scaling}, and \ref{visuals}.

\section{Ablation Studies}
\label{ablations}
We analyze the impact of key design choices in \tabref{tab:ablation_results}.
\textbf{\method{SIREN}} removes the residual connections from our \method{ReSIREN} blocks, consistently resulting in a performance drop, with $-1.2$ R$^2$ for plant traits regression.
Appendix \ref{scaling} further highlights the impact of residual connections with increasing model depth.
To explore the impact of learning temporal climatic data representations, \textbf{\method{All- Months}} is pretrained to reconstruct the concatenation of all monthly \method{CHELSA} variables at each location, while \textbf{\method{March-Only}} is pretrained exclusively on March data ($m_i = 3$). 
Our results support support the choice of making \method{Climplicit} a spatio-temporal location encoder rather than encoding only a single month or the entirety of yearly climatic variables.
In \textbf{\method{No H-SIREN}} we remove the \method{H-SIREN} activation function and observe a performance drop on two of our tasks.
\textbf{\method{rec-CHELSA}} runs single-layer probing on the reconstructed \method{CHELSA} variables instead of the embeddings before the projection layer. 
As expected, the latent representation better generalizes compared to reconstructed climatic variables.
In \textbf{\method{CHELSA-CLIP}}, we pretrain our model with a contrastive objective similar to \cite{klemmer2023satclip} which proves to produce poorer representations compared to directly regressing \method{CHELSA}.
Whether this approach extends to more complicated regressions or auto-encoding settings for multi-pixel images is left to future work. 
Finally, \textbf{\method{ERA5}} compares representations learned from monthly \textbf{\method{ERA5}} \citep{dee2011era} instead of \method{CHELSA}.
We find that the $\times30$ higher resolution of \method{CHELSA} ($30$ versus $900$ arcsec) allows for learning richer representations for our marcoecological tasks.

\begin{table}[!h]
    \caption{
    Ablation analysis with single-layer probing.
    Mean and standard deviation of ten random initializations of the single-layer probe.
    \textbf{Best}.
    }
    \centering
    \begin{tabular}{l  c  c  c  } 
    \toprule
    \textbf{Model} & Biomes (\% F1 $\uparrow$) & SDM (\% Acc $\uparrow$) & Plant traits (\% R² $\uparrow$)  \\ [0.5ex] 
    \midrule
    \method{Climplicit}      & \textbf{78.4} $\pm$ 0.3   &  3.2   $\pm$ 0.0   &  \textbf{70.0} $\pm$ 0.1 \\[1ex] 
    \midrule
    \method{SIREN}          & 77.5 $\pm$ 0.2     &  3.1   $\pm$ 0.0   &  68.8 $\pm$ 0.2\\
    \method{Concat Months}  & 75.9 $\pm$ 0.3     &  2.6   $\pm$ 0.0   &  66.0 $\pm$ 0.1   \\
    \method{March-Only}     & 78.2 $\pm$ 0.2     &  2.9   $\pm$ 0.0   &  62.8 $\pm$ 0.1   \\
    \method{No H-SIREN}     & 77.9 $\pm$ 0.2     &  \textbf{3.6}   $\pm$ 0.0   &  69.1 $\pm$ 0.1   \\
    \method{rec-CHELSA}     & 61.5 $\pm$ 0.2     &  1.5   $\pm$ 0.0   &  55.4 $\pm$ 0.1   \\
    \method{CH-CLIP}        & 76.5 $\pm$ 0.6     &  2.3   $\pm$ 0.1   &  66.9 $\pm$ 0.4   \\
    \method{ERA5}           & 63.7 $\pm$ 0.5     &  1.9   $\pm$ 0.1   &  68.6 $\pm$ 0.2   \\ 

    \bottomrule
    \end{tabular}
    \label{tab:ablation_results}
    \vspace{-0.05in}
\end{table}

\section{Discussion and Conclusion}
We introduced \method{Climplicit}, a spatiotemporal encoder that produces implicit climatic representations anywhere on Earth, aimed at lowering the storage, compute, and machine learning expertise requirements for downstream practitioners.
Our model hinges on \method{ReSIREN}, a residual extension of \method{SIREN} that allows to learn high-resolution representations.
Our \method{Climplicit} model occupies $99.9\%$ less memory than the \method{CHELSA} dataset it was trained on, while producing representations that perform on par or better than from-scratch training on several macroecological tasks. 
Potential limitations of \method{Climplicit} embeddings are their \textit{implicit} nature and lower resolution, compared to the original climatic raster dataset.
Finally, in Appendix \ref{fut_clim}, we outline how \method{Climplicit} could be extended to provide embeddings for future climatologies under different climate change scenarios, in turn allowing future predictions for the downstream tasks.

\section{Acknowledgments}
This project was partially funded by the Embed2Scale project, co-funded by the EU Horizon Europe Programme under Grant Agreement No. 101131841. Additional support was provided by the Swiss State Secretariat for Education, Research and Innovation (SERI) and UK Research and Innovation (UKRI).

{
    \small
    \bibliographystyle{abbrvnat}
    \bibliography{main}
}

\newpage
\appendix

\section{Impact and Ethical Statement}
This work is part of an ongoing project co-designed with ecologists to make deep learning methods more easily accessible to practitioners. 
Future work will apply the model to more extensive ecological problems. 
After finishing pretraining, this model has small requirements in regards to compute, storage capacity and know-how, which increases ease of use for non-deep learning experts and reduces the carbon footprint of downstream applications. 
No GenAI such as ChatGPT was used in the making of this work and writing of this publication.

\section{\method{ReSIREN} Details}

\subsection{\method{ReSIREN} Activations}\label{activations}
Let $L_j, j \in \{1, ..., D\}$ be the j-th layer of the network of depth $D$, consisting of an affine transformation $F_j$ and activation function $A_j$. 
The activation functions across the network then are:
$$
A_j(x)=
\begin{cases}
\text{sin}(\omega_0\text{sinh}(2x)) & \text{for j = 1}\\
\text{sin}(\omega_0x) & \text{for 1 < j < D}\\
\text{Identity}(x), & \text{for j = D}
\end{cases}
$$
Following the results of \cite{sitzmann2020implicit}, we set $\omega_0 = 30$. 
According to the authors, $\omega_0$ increases the input range across multiple periods of sin, which allows the network to converge more quickly and access more frequencies. 
The special case in the first layer is the H-SIREN function of \cite{gao2024h}.

\subsection{Distributional Stability of \method{ReSIREN}}
\label{stability}

\begin{figure}[!h]
\centering
    \includegraphics[width=0.8\linewidth]{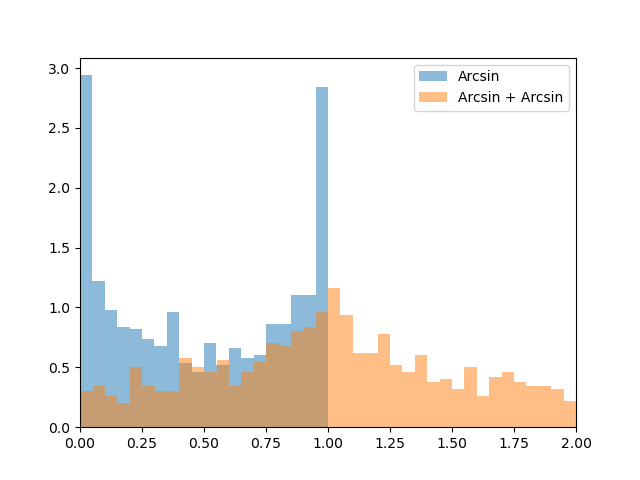}
    \caption{
    Experimental probability density function of adding up two independent $\text{Arcsin}(0,1)$. 
    The sum is a ill-behaved distribution that can not be trivially normalized back to $\text{Arcsin}(0,1)$. 
    Therefore, we choose to apply residual connections between to the well-behaved Gaussians in the middle of each \method{SIREN} layer.}
    \label{arcsin_stability}
\end{figure}

Let $z_j$ be the input to $L_j$, $h_j = F_j(z_j)$ the hidden pre-activation vector and $z_{j+1} = L_j(z_j)$ the output of the layer.
According to \cite{sitzmann2020implicit}, the distribution of $z_j$ is $\text{Arcsin}(-1,1)$ and $\mathcal{N}(0,1)$ for $h_j$. 
These distributional properties are essential to the functioning of the network and are enabled by a specific initialization of the weights of $F_j$. 
Adding residual connections the classic way by adding the output of the previously layer to the current layer's output would amount to $z'_{j+1} = z'_j + z_{j+1}$, but the sum of two Arcsin distributions turns into a badly behaved distribution (See \figref{arcsin_stability}).
Therefore, we introduce residual connections between the hidden vectors $h$, specifically as $h'_j = (h_j + h'_{j-1}) / 2$ with $z_{j+1} = A_j(h'_j)$.

\begin{lemma} 
Let $X$,$Y \sim \mathcal{N}(0,1)$ be independent and $Z = (X+Y)/2$. 
Then $Z \sim \mathcal{N}(0,1)$.
\end{lemma}

\begin{proof} 
The sum of two independent normally distributed variables is equally normally distributed, with the mean and variance added up. 
This results in $\mathcal{N}(0+0,1+1)$. 
This is regularized to $\mathcal{N}(0,1)$ by dividing through the variance $2$.
\end{proof}

We assume independence between the two normally distributed vectors $h_j$ and $h'_{j-1}$. 
With the residual sum being normally distributed again, we thus provide the requirements for \textbf{Lemma 1.6.} of \cite{sitzmann2020implicit} on keeping the activation distributions stable across the network.

\section{Training details}\label{training}
For each epoch, we iterate once over all land locations and sample a random month $m_i$ uniformly for each location. 
We train the network up to 20 epochs with the Adam optimizer and early stopping, totaling up to $547$ thousand training steps. 
We are learning a dense representation across the globe with a large number of datapoints and therefore do not require a val/test split.
We use a model depth of 16, hidden size of 512, embeddings size of 256, and batch size of 8192 pixels.
Although we found that a higher embedding size may slightly increase performance, we stick to 256 to align with similar geospatial encoding models.
Our pretraining converges after $\sim$ 7 hours on a NVIDIA T4 with 16 GB of RAM and 16 CPUs. 

\section{Downstream Tasks}\label{sec:dsts}

\paragraph{Biomes classification.}
We reproduce the biomes classification task from \citep{klemmer2023satclip} by sampling 100k land locations uniformly across the globe and then extracting the biome at each location from the terrestrial ecoregions of the world map \citep{olson2001terrestrial}. 
The train-val-test split is $[0.5, 0.1, 0.4]$. 
As metric we use the mean of the F1 scores for all biomes, dubbed "macro" as it gives equal weight to all classes.

\paragraph{Species distribution modeling.}
In species distribution modeling (SDM), we aim to predict if a species is likely to be observed at a location in a given month. 
We use the preprocessed \citep{dollinger2024sat} presence-only data from the GLC23 challenge \citep{botella2023geolifeclef}, encompassing 2 million occurrences for 5.6 thousand plant species across Europe. 
The train-val-test split is $[0.7, 0.05, 0.25]$. 
We train the probe with the AN-full loss with background sampling \citep{cole2023spatial}, and report the Top-1 Accuracy.

\paragraph{Plant traits regression.}
We sample 100k locations as in the biomes classification and extract means and standard deviations for 4 plant functional traits from the global plant trait maps provided by \cite{moreno2018methodology}. The traits are: specific leaf area, leaf dry matter content, leaf nitrogen and phosphorus content per dry mass, and leaf nitrogen/phosphorus ratio.
We split the data $[0.5, 0.1, 0.4]$ and report the mean R² score across the eight values.

SDM has a date of observation associated with each occurrence, thus we sample \method{Climplicit} for the embedding of the respective month. In the case of biomes and plant traits, the rasters we sample from are immutable across the year. Here, we retrieve the embeddings for March, June, September and December from \method{Climplicit} and concatenate them as input to the probe for the classification and regression tasks.

\section{Future Work: Extension to Future Climatologies}\label{fut_clim}
CHELSA has been extended to include four monthly precipitation and temperature rasters for the climatologies 2011-2040, 2041-2070 and 2071-2100 \citep{karger2020high}. 
These are calculated from four different global circulation models for the three shared socioeconomic pathways SSP126, SSP370 and SSP585 \citep{gidden2019global}.

A future extension of \method{Climplicit} may include these future climatologies by first declaring a categorical variable $c_i \in \{-1, -0.33, 0.33, 1\}$ where each value corresponds to one of the four thirty-year intervals. This variable may be added to the positional encoding as another temporal variable:
\begin{equation*}
    p_i = \Big[ \: \lambda_i, \: \phi_i, \: \text{sin}(2\pi *m_i / 12), \: \text{cos}(2\pi* m_i / 12)\:, c_i \: \Big] \; \in \; [-1, 1]^{5} \; .
\end{equation*}

This may enable the location encoder to encode all four climatologies in a shared embedding space. 
It is important to note that for this method, not all eleven CHELSA variables from the contemporary climatology can be used during training, but only the four also provided for the future climatologies. 
This, along with the increased learning requirements from three additional climatologies might decrease performance on the downstream task setup of this work.

Having now learned a shared embedding space for the four climatologies, a downstream task may now train on the contemporary climatology $c_i = -1$, and make predictions on the future climatologies $c_i \in \{-0.33, 0.33, 1\}$. This can be seen as an downstream extrapolation to an un-sampled area, which the embeddings produced by location encoders are especially well suited to \citep{klemmer2023satclip}.

\section{Extended Results}

\begin{table}[!h]
    \caption{
    Single-layer probing on diverse geolocation representations for non-ecological downstream tasks.
    Mean and standard deviation of ten random initializations of the single-layer probe.
    \textbf{Best}.
    \underline{Best pretrained geolocation representation}.
    }
    \centering
    \begin{tabular}{l  c  c  c } 
     \toprule
     \textbf{Model} & Pop Density (\% R² $\uparrow$) & Med Income (\% R² $\uparrow$) & Cali Housing (\% R² $\uparrow$)  \\ [0.5ex] 
     \midrule
     \method{FS Loc       }  & \textbf{75.3} $\pm$ 5.1    &  -7.5 $\pm$ 17.2  &  18.9 $\pm$ 18.9    \\ 
     \method{FS CH        }  & 64.9 $\pm$ 2.1    &  29.8 $\pm$ 2.9  &  \textbf{61.9} $\pm$ 2.2     \\ 
     \method{FS Loc $+$ CH  }  & 75.1 $\pm$ 3.4    & 23.3 $\pm$ 13.5      &  60.3 $\pm$ 3.5    \\ 
     \midrule
     \method{SatCLIP         }  & 49.1 $\pm$ 0.5       &  30.0 $\pm$ 9.1      &  35.2 $\pm$ 0.4   \\
     \method{Taxabind        }  & 41.0 $\pm$ 0.2       &  15.0 $\pm$ 0.1      &  44.2 $\pm$ 0.1   \\
     \method{SINR            }  & 38.8 $\pm$ 0.3      &  38.6 $\pm$ 0.5      &  37.1 $\pm$ 0.3    \\
     \method{CSP             }  & 35.8 $\pm$ 0.7       &  27.3 $\pm$ 1.5      &  50.1 $\pm$ 0.6   \\
     \method{GeoCLIP         }  & 39.3 $\pm$ 0.1       &  34.7 $\pm$ 0.1&  \underline{58.9} $\pm$ 0.1   \\
     \midrule
     \method{Climplicit}  & \underline{67.0} $\pm$ 0.3 & \textbf{\underline{45.0}} $\pm$ 0.4 & 44.2 $\pm$ 0.6 \\[1ex]  

     \bottomrule
    \end{tabular}
    \label{tab:non_eco_results}
    \vspace{-0.05in}
\end{table}

\subsection{Non-Ecological Downstream Tasks}
\label{non_eco_section}
We furthermore report in \tabref{tab:non_eco_results} the results on three non-ecological downstream tasks adapted from \cite{klemmer2023satclip}.
We sample 100k \textbf{population density} values in the United States of America from \cite{doxsey2015taking}.
For \textbf{median income} regression, we use the 2021 US county-level median household income data from the US Department of Agriculture\footnote{https://www.ers.usda.gov/data-products/county-level-data-sets/county-level-data-sets-download-data}.
\textbf{California housing} is a long-established dataset from \cite{pace1997sparse} of median house prices in California census blocks. 
On all tasks we report the R$^2$ score.

\method{Climplicit} is the best location encoder on two of the three tasks, also managing to outperform the fully trained models on the challenging median income task. 
Only on the California housing dataset does \method{Climplicit} fall short, which may be attributed to the high variation in housing prices in California independent of the local variation in climate.

\subsection{MLP Results}
\label{mlp_exps}
\begin{table}[ht]
    \caption{
    MLP probing on diverse geolocation representations.
    Mean and standard deviation of ten random initializations of the linear probe.
    \textbf{Best}.
    \underline{Second best}.
    }
    \centering
    \begin{tabular}{l  c  c  c  } 
     \toprule
     \textbf{Model} & Biomes (\% F1 $\uparrow$) & SDM (\% Acc $\uparrow$) & Plant traits (\% R² $\uparrow$)  \\ [0.5ex] 
     \midrule
     \method{FS Loc       }  & 73.1 $\pm$ 3.2   &  1.64 $\pm$ 0.60  &  70.6 $\pm$ 0.9  \\ 
     \method{FS CH        }  & 66.6 $\pm$ 1.8    &  3.16 $\pm$ 0.08  &  74.4 $\pm$ 1.2   \\ 
     \method{FS Loc $+$ CH  }  & \underline{75.7} $\pm$ 3.1    &  3.17 $\pm$ 0.11      &  76.4 $\pm$ 0.9  \\ 
     \midrule
     \method{SatCLIP       }  & 72.3 $\pm$ 0.7       &  2.80 $\pm$ 0.07       &  \underline{76.9} $\pm$ 0.3    \\
     \method{Taxabind        }  & 63.6 $\pm$ 1.1       &  3.04 $\pm$ 0.07      &  71.4 $\pm$ 0.3    \\
     \method{SINR            }  & 64.4 $\pm$ 0.8       &  2.68 $\pm$ 0.07      &  72.8 $\pm$ 0.7    \\
     \method{CSP             }  & 64.2 $\pm$ 1.3       &  3.06 $\pm$ 0.05      &  72.0 $\pm$ 0.4    \\
     \method{GeoCLIP         }  & 64.9 $\pm$ 0.6       &  \underline{3.20} $\pm$ 0.06      &  71.6 $\pm$ 0.3    \\
     \midrule
     \method{Climplicit} (Ours)  & \textbf{78.2} $\pm$ 0.7 & \textbf{3.70} $\pm$ 0.04 &\textbf{78.6} $\pm$ 0.3\\[1ex]
     \bottomrule
    \end{tabular}
    \label{tab:mlp_results}
    \vspace{-0.05in}
\end{table}
In Table \ref{tab:mlp_results}, instead of using a single layer as probe, we also calculate results when training a multi-layer perceptron consisting of three hidden layers with 64 neurons on the various embeddings. This is similar to the downstream setup of \cite{klemmer2023satclip}, although the shape of this MLP has not been fine-tuned. We can see that \method{Climplicit} outperforms all comparison methods in this setup.

\subsection{CHELSA Reconstruction}
\label{chelsa_rec}

\begin{figure}[!h]
    \centering
    
    \begin{subfigure}{.95\textwidth}
    \centering
    \includegraphics[width=1\linewidth]{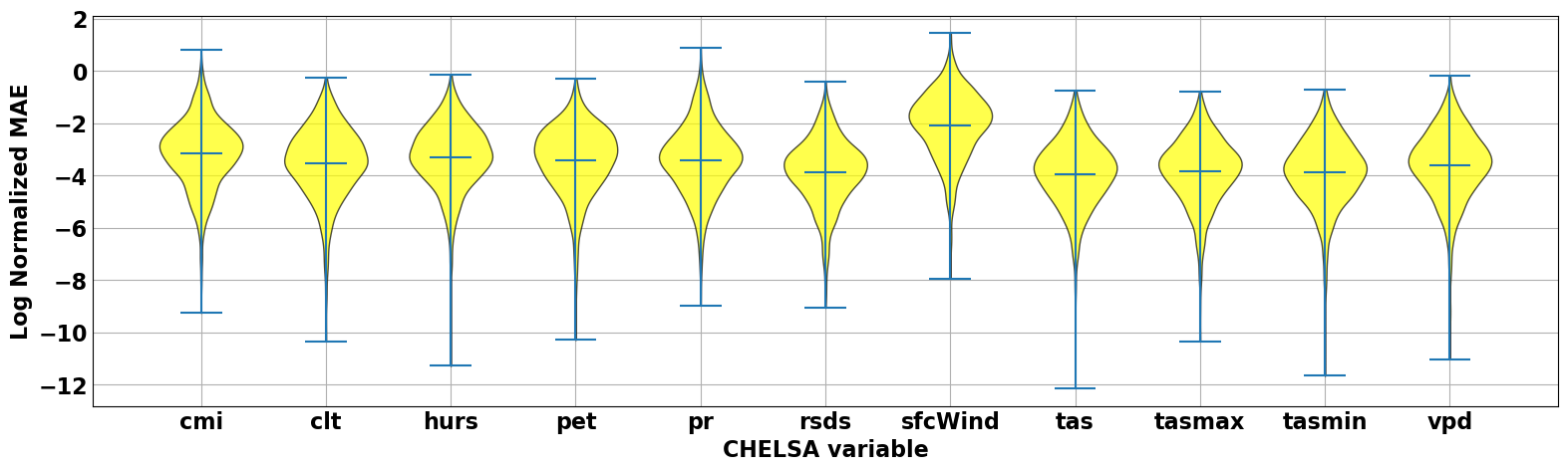}
    \label{fig:variable_errors}
    \end{subfigure}

    \begin{subfigure}{.95\textwidth}
    \centering
    \includegraphics[width=1\linewidth]{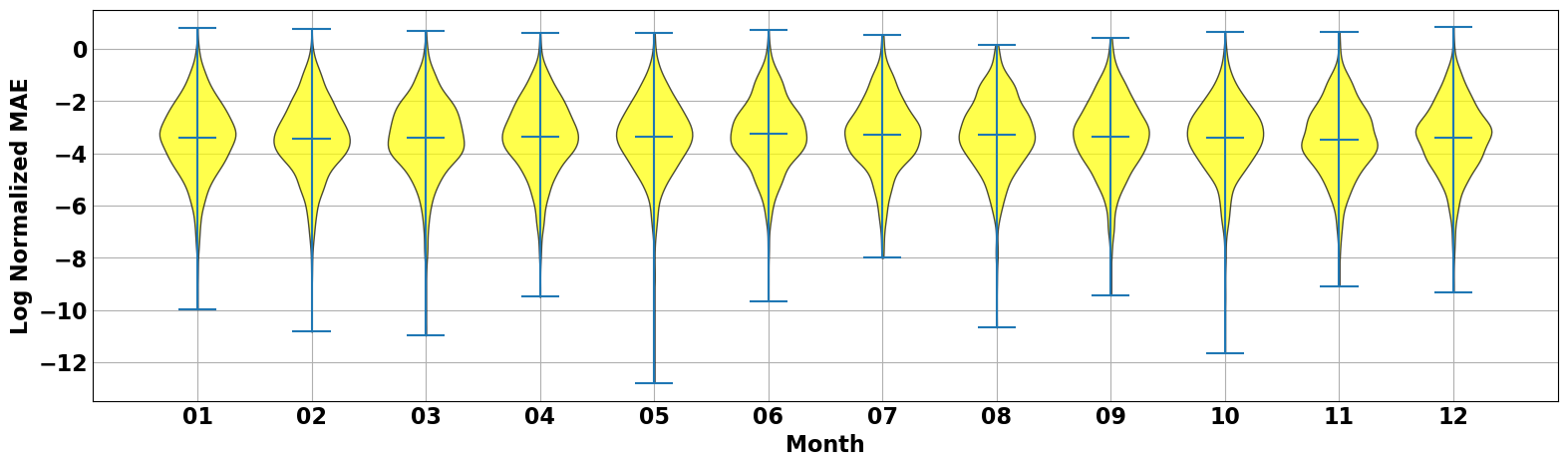}
    \label{fig:month_errors}
    \end{subfigure}

    \begin{subfigure}{.95\textwidth}
    \centering
    \includegraphics[width=1\linewidth]{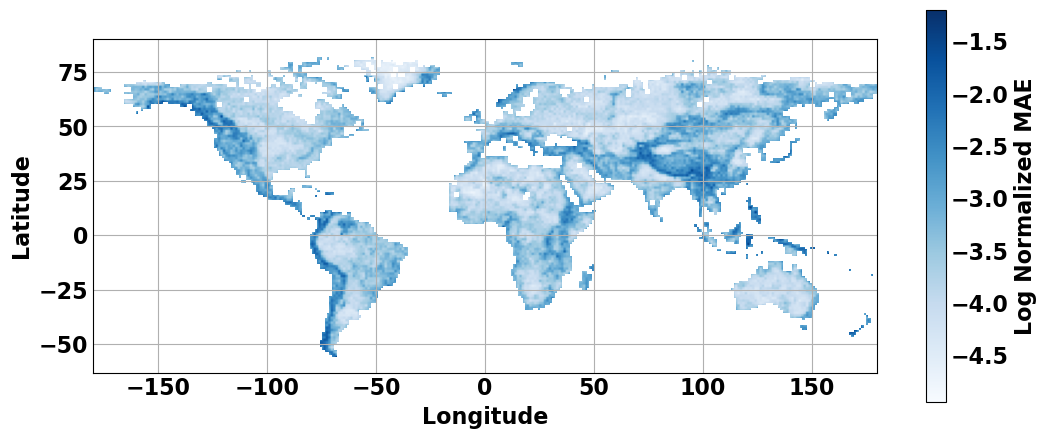}
    \label{fig:global_errors}
    \end{subfigure}
    \caption{Analysis of the reconstruction error of CHELSA after the \method{Climplicit} pretraining. \textbf{Top}: Absolute error distribution for each climatic variable. \textbf{Middle}: Absolute error distribution for each month. \textbf{Bottom}: Mean absolute error inside 136 by 320 grid cells across the globe.}
    \label{fig:errors}
\end{figure}
We samples 100k locations across earth and calculate the absolute error between the original CHELSA values and the reconstruction by \method{Climplicit} for each location and month.
Note that we calculate this in the Gaussian-normalized space to keep values comparable between the different variables.
We also apply log to make the visualizations less sensitive to extreme values.

In Figure \ref{fig:errors} we observe that \method{Climplicit} reconstructs temperature and radiation the best, but struggles significantly with wind speeds.
There is no discernible variation across months.
Spatially, the error is concentrated in mountainous areas.
As CHELSA is a down sampling of ERA5 data based on a digital elevation model, it follows that the local variation in CHELSA is strongly correlated with elevation changes.
Thus the smooth reconstruction by \method{Climplicit} performs well in flat areas where CHELSA is equally smooth, but struggles in areas with significant elevation changes.
Future work can tackle this issue by oversampling areas with high elevation variability to focus the models capacity there.

\subsection{Scaling of \method{ReSIREN}}
\label{scaling}

\begin{figure}[!h]
    \centering
    \begin{subfigure}{.47\textwidth}
    \centering
    \includegraphics[width=1\linewidth]{figures/images/biomes_macro_f1.png}
    \label{fig:scaling_Climplicit}
    \end{subfigure}
    \begin{subfigure}{.47\textwidth}
    \centering
    \includegraphics[width=1\linewidth]{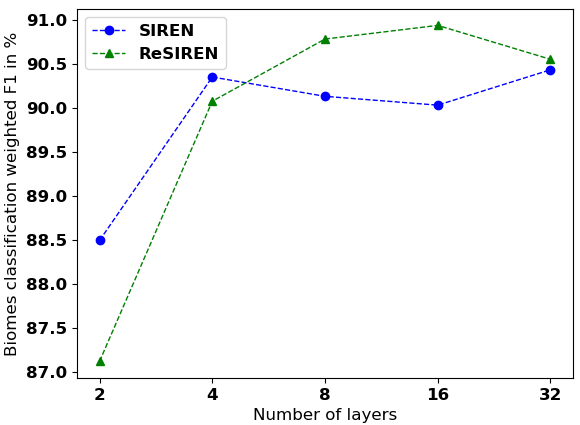}
    \label{fig:scaling_satclip}
    \end{subfigure}
    \newline
    
    \begin{subfigure}{.47\textwidth}
    \centering
    \includegraphics[width=1\linewidth]{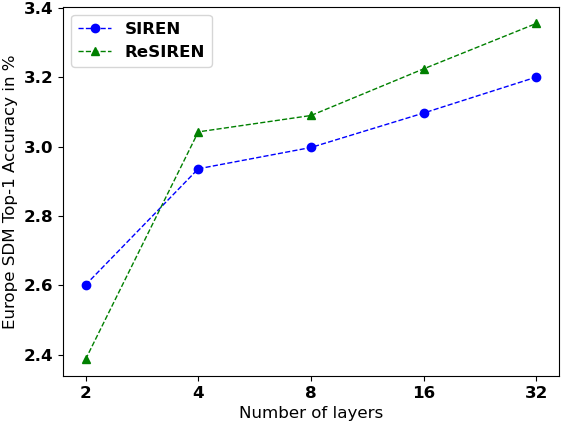}
    \label{fig:sdm_scaling}
    \end{subfigure}
    \begin{subfigure}{.47\textwidth}
    \centering
    \includegraphics[width=1\linewidth]{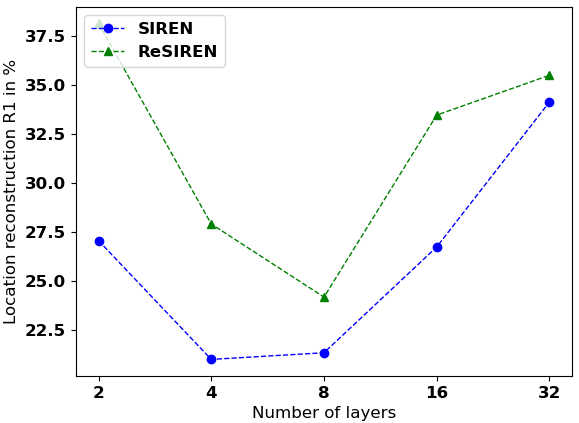}
    \label{fig:scaling_loc_r1}
    \end{subfigure}
    \newline
    
    \begin{subfigure}{.47\textwidth}
    \centering
    \includegraphics[width=1\linewidth]{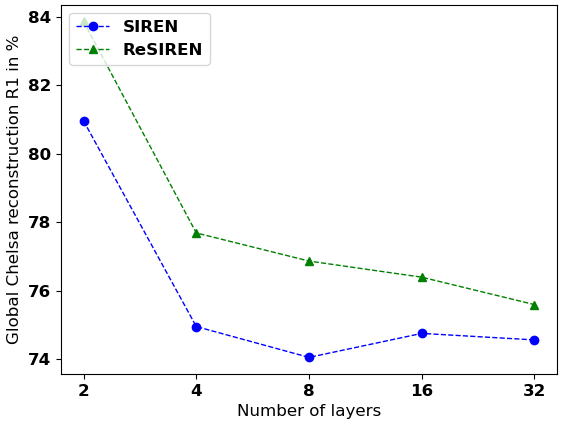}
    \label{fig:scaling_global_chelsa_r1}
    \end{subfigure}
    \begin{subfigure}{.47\textwidth}
    \centering
    \includegraphics[width=1\linewidth]{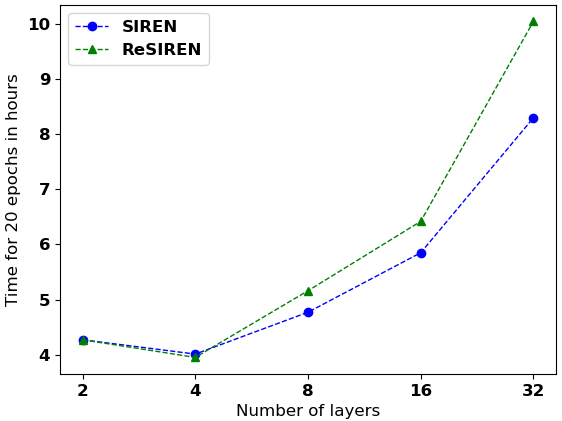}
    \label{fig:scaling_trainingtime}
    \end{subfigure}
\caption{Comparison of the scaling behavior of \method{ReSIREN} compared to \method{SIREN}. Generally, \method{ReSIREN} performs worse at few layers, but outperforms and scales better with more layers in exchange for an increase in training time.}
\label{fig:scaling}
\end{figure}

In \figref{fig:scaling} we report both the macro and weighted F1 score, along with the European SDM accuracy. 
Location and global CHELSA reconstruction are regression tasks where location have been sampled across the globe similar to the biomes task. 
For a location, either the location itself or the respective CHELSA vector was then regressed. 
Lastly, we present the increase in runtime from using the residual connections.

\subsection{Qualitative Results}\label{visuals}

\begin{figure}[!h]
    \centering
    \begin{subfigure}{.45\textwidth}
    \centering
    \includegraphics[width=1\linewidth]{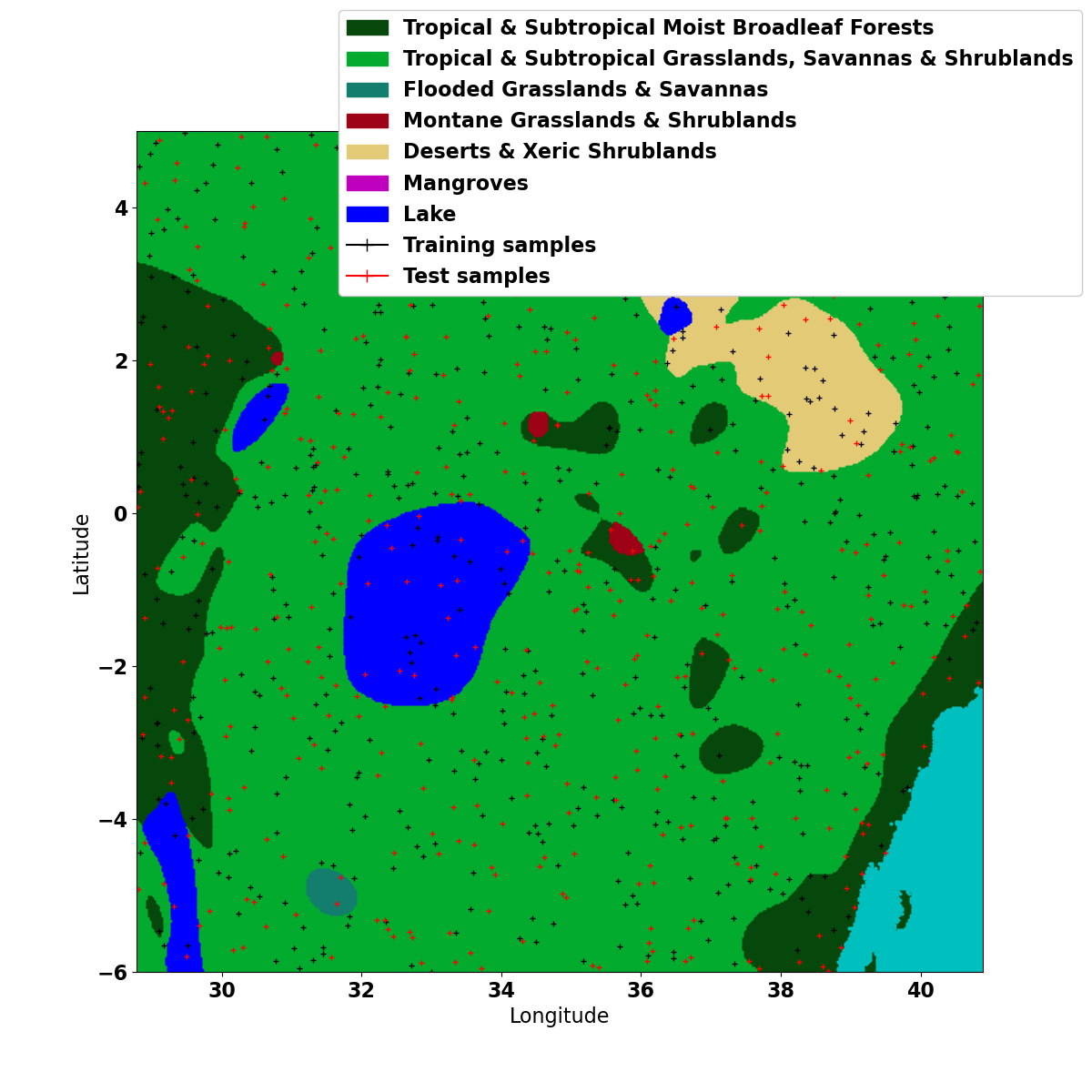}
    \vspace{-0.8cm}
    \caption{Climplicit}
    \label{fig:lvb_Climplicit}
    \end{subfigure}
    \begin{subfigure}{.45\textwidth}
    \centering
    \includegraphics[width=1\linewidth]{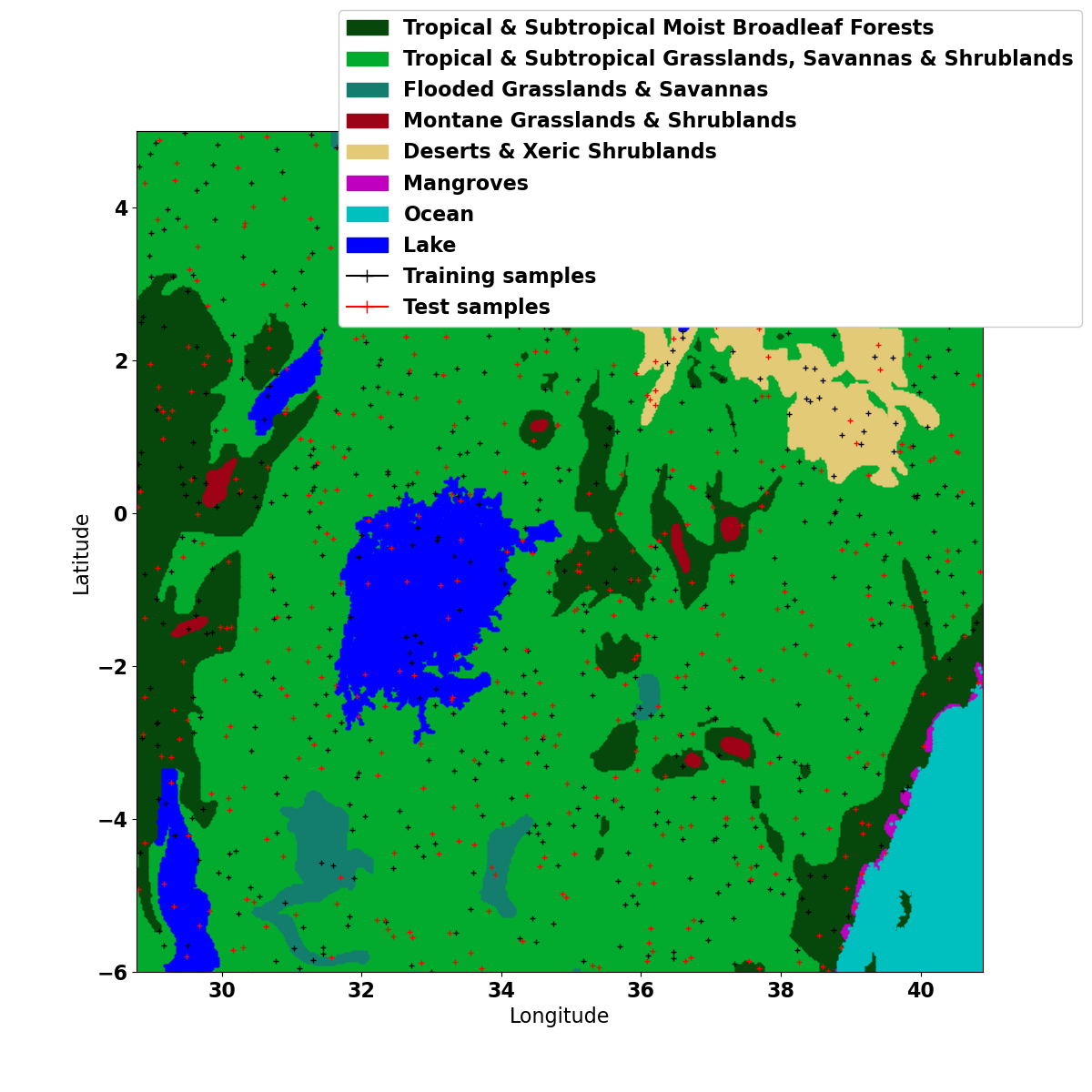}
    \vspace{-0.8cm}
    \caption{Groundtruth}
    \label{fig:lvb_gt}
    \end{subfigure}
    \newline

    \begin{subfigure}{.45\textwidth}
    \centering
    \includegraphics[width=1\linewidth]{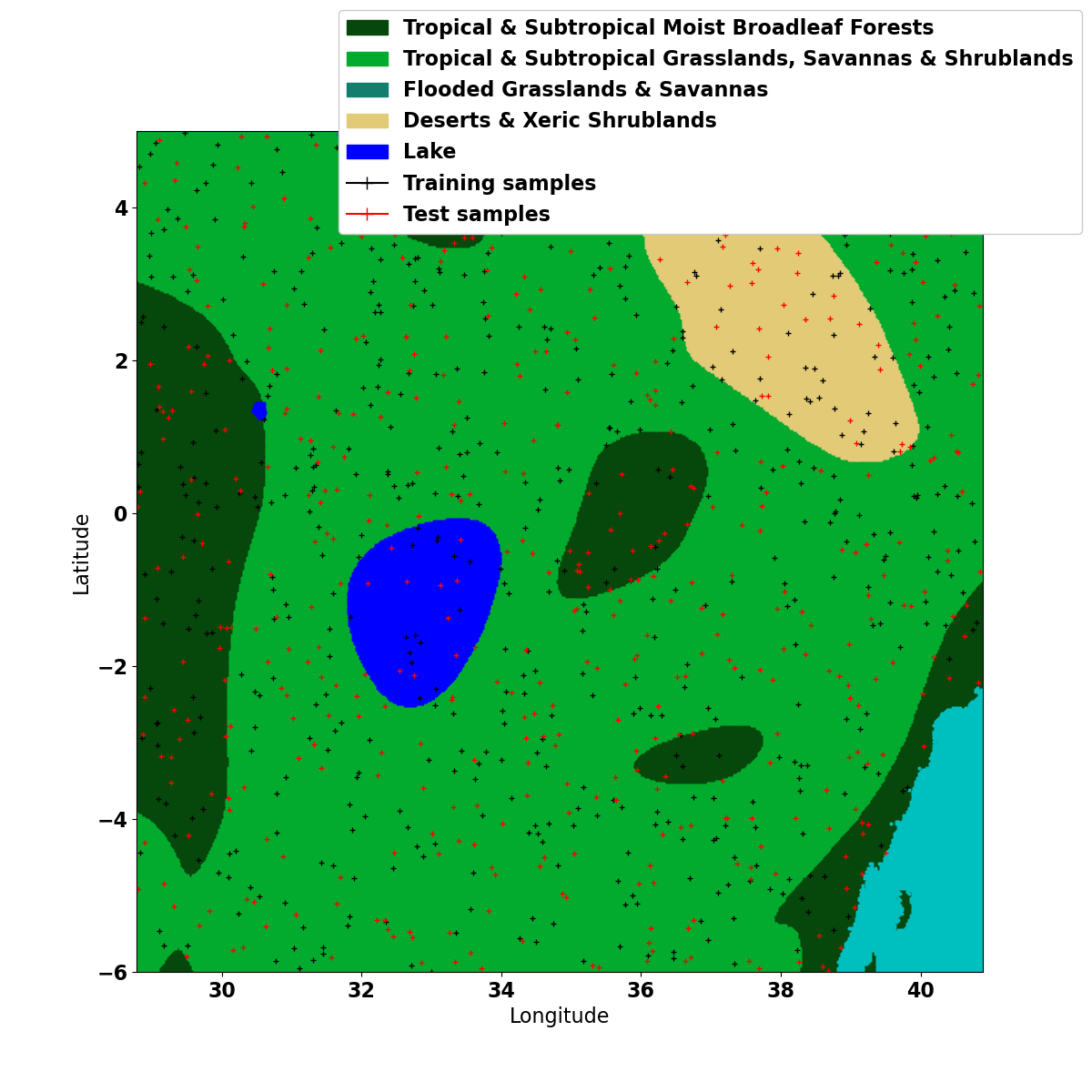}
    \vspace{-0.8cm}
    \caption{FS LOC}
    \label{fig:lvb_fs_loc}
    \end{subfigure}
    \begin{subfigure}{.45\textwidth}
    \centering
    \includegraphics[width=1\linewidth]{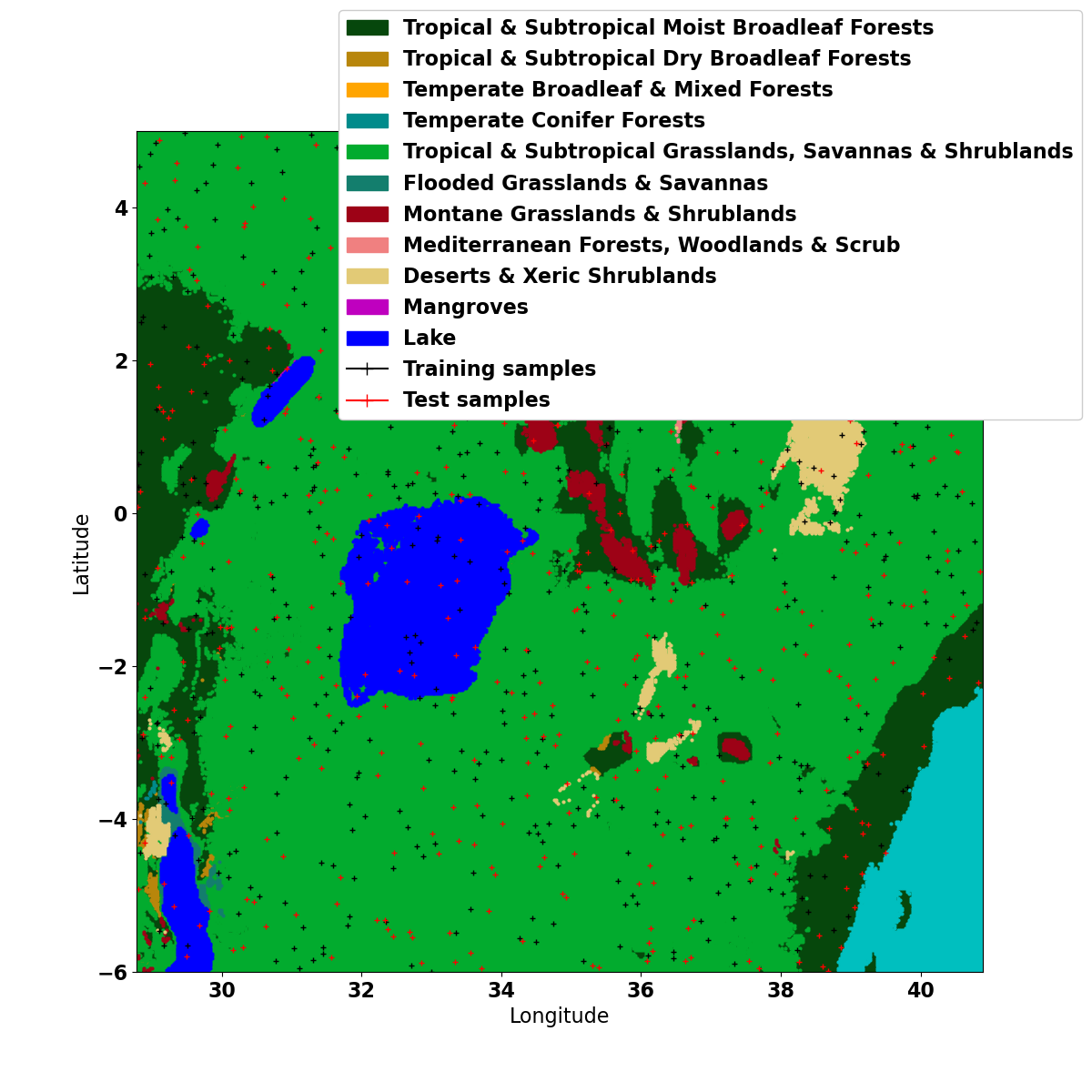}
    \vspace{-0.8cm}
    \caption{FS CH}
    \label{fig:lvb_fs_ch}
    \end{subfigure}
    \newline
    \begin{subfigure}{.45\textwidth}
    \centering
    \includegraphics[width=1\linewidth]{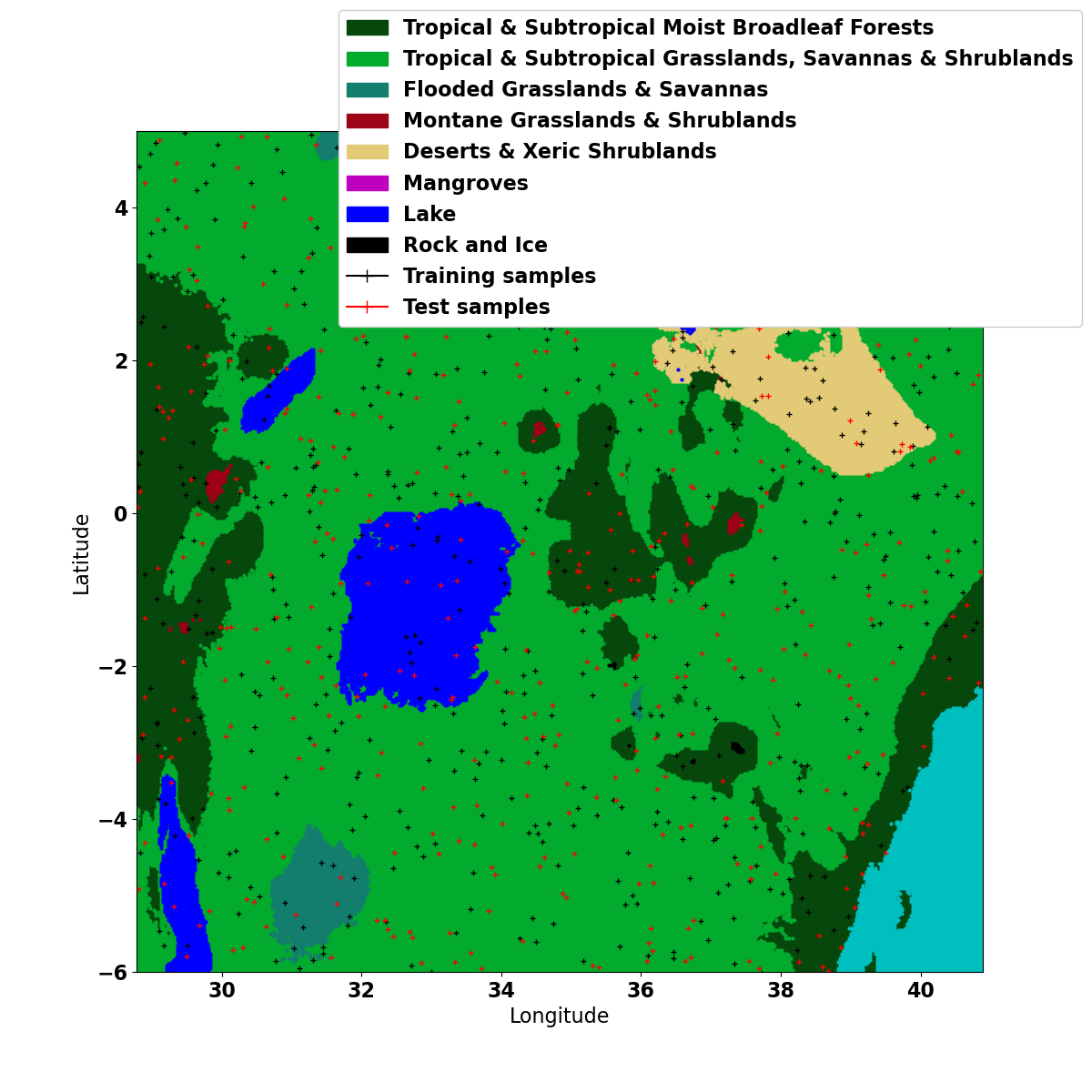}
    \vspace{-0.8cm}
    \caption{FS LOC $+$ CH}
    \label{fig:lvb_fs_loc_ch}
    \end{subfigure}
    \begin{subfigure}{.45\textwidth}
    \centering
    \includegraphics[width=1\linewidth]{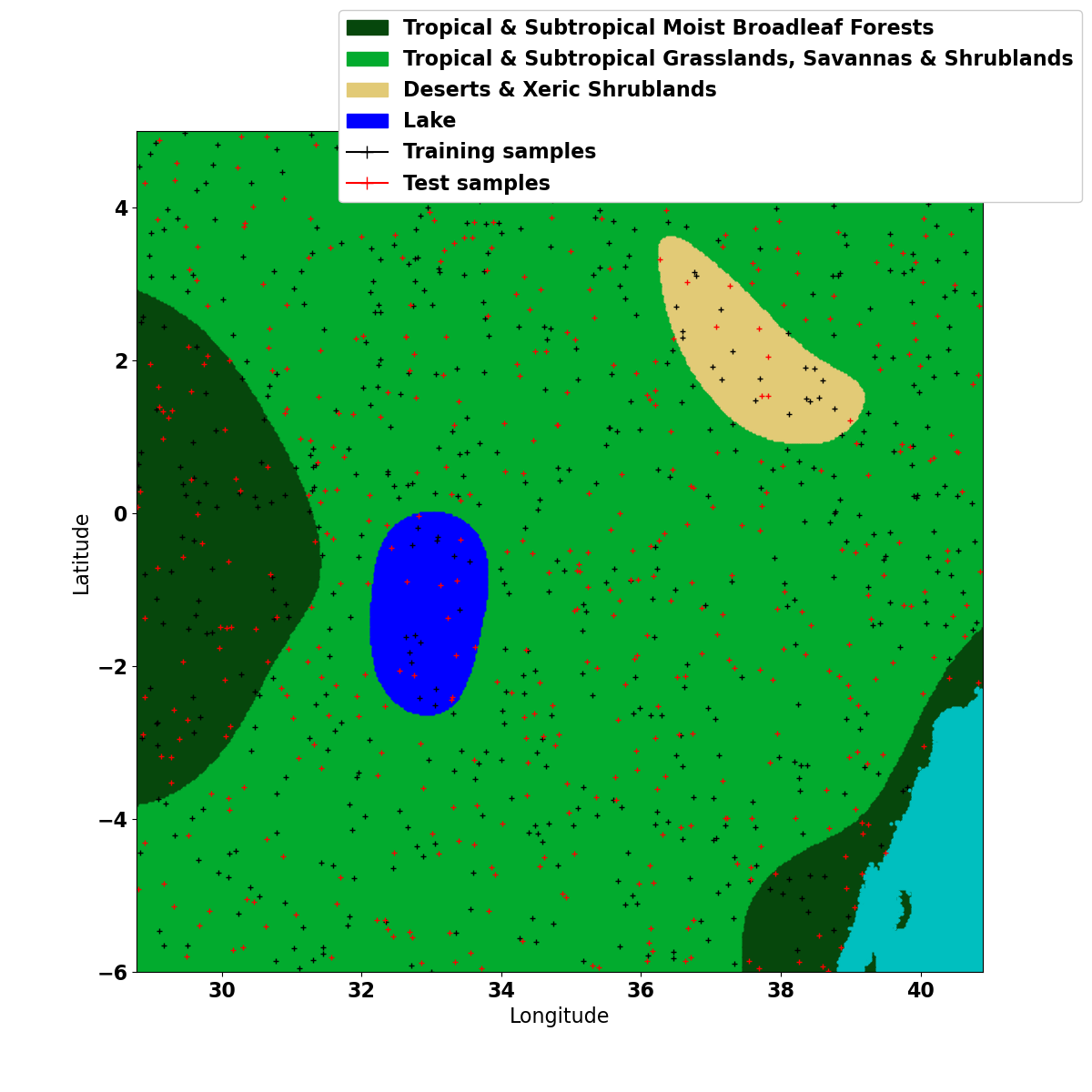}
    \vspace{-0.8cm}
    \caption{SatCLIP}
    \label{fig:lvb_satclip}
    \end{subfigure}
    \newline
\end{figure}

\begin{figure}
\centering
\ContinuedFloat
\begin{subfigure}{.45\textwidth}
    \centering
    \includegraphics[width=1\linewidth]{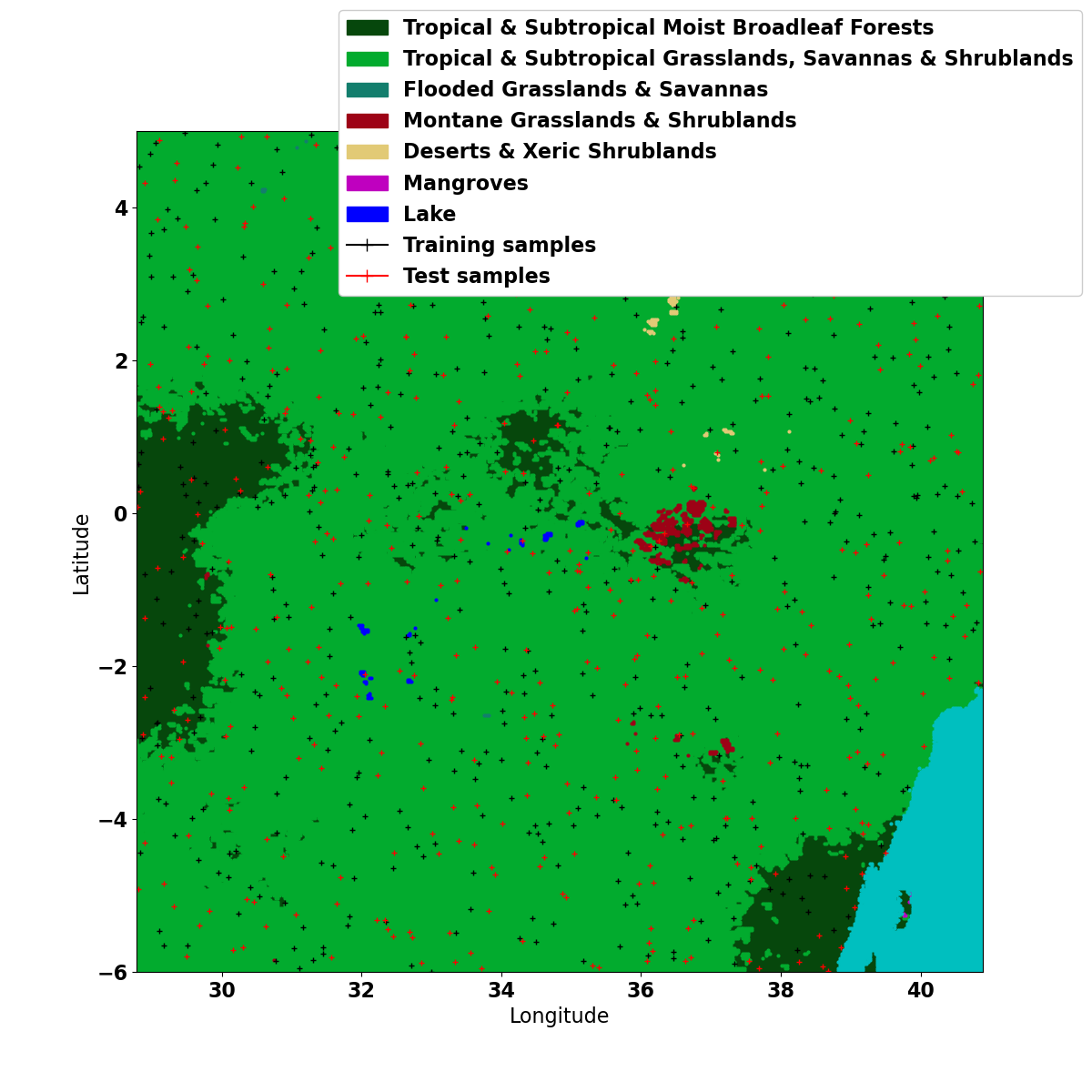}
    \vspace{-0.8cm}
    \caption{GeoCLIP}
    \label{fig:lvb_geoclip}
    \end{subfigure}
    \begin{subfigure}{.45\textwidth}
    \centering
    \includegraphics[width=1\linewidth]{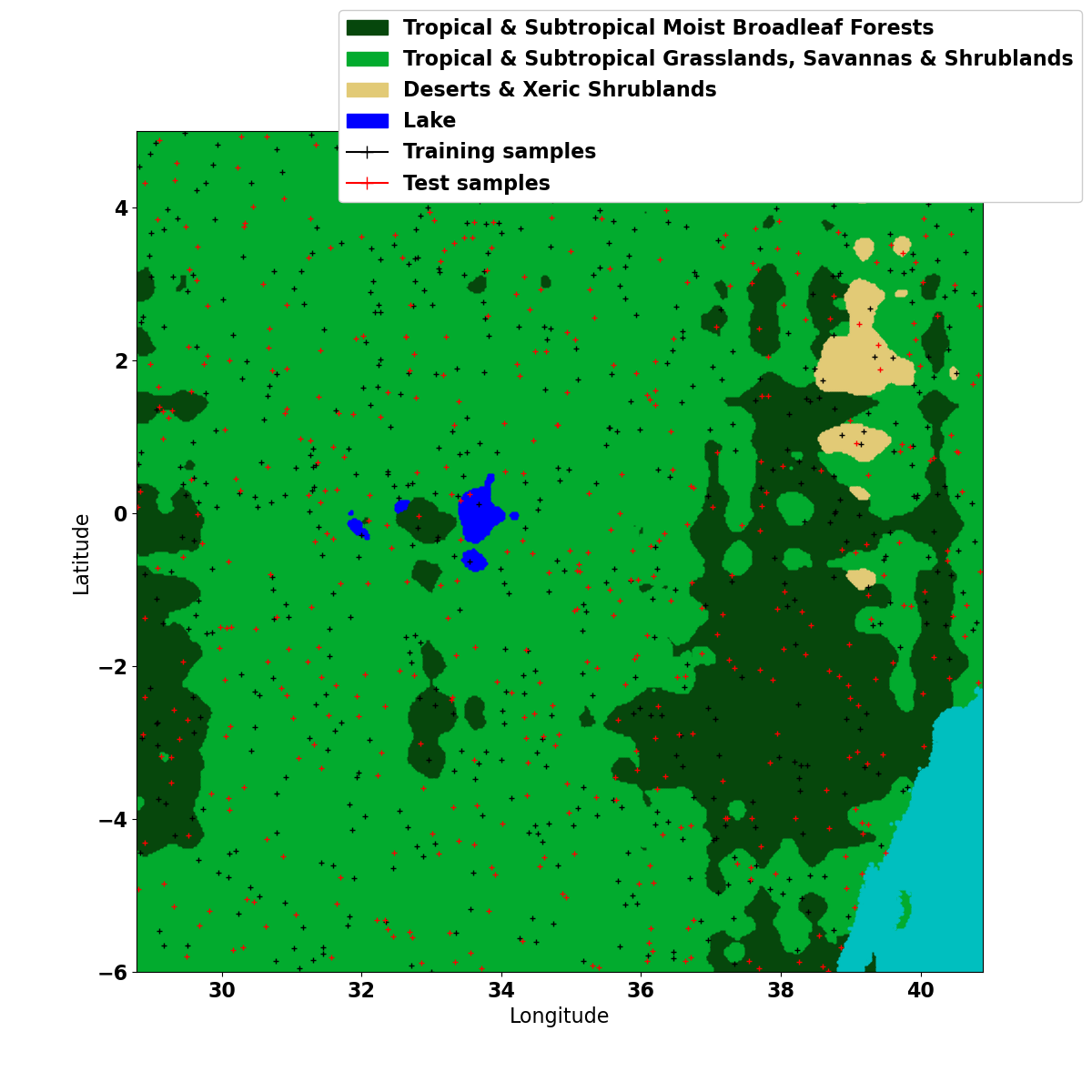}
    \vspace{-0.8cm}
    \caption{CSP}
    \label{fig:lvb_csp}
    \end{subfigure}
    \newline
    \begin{subfigure}{.45\textwidth}
    \centering
    \includegraphics[width=1\linewidth]{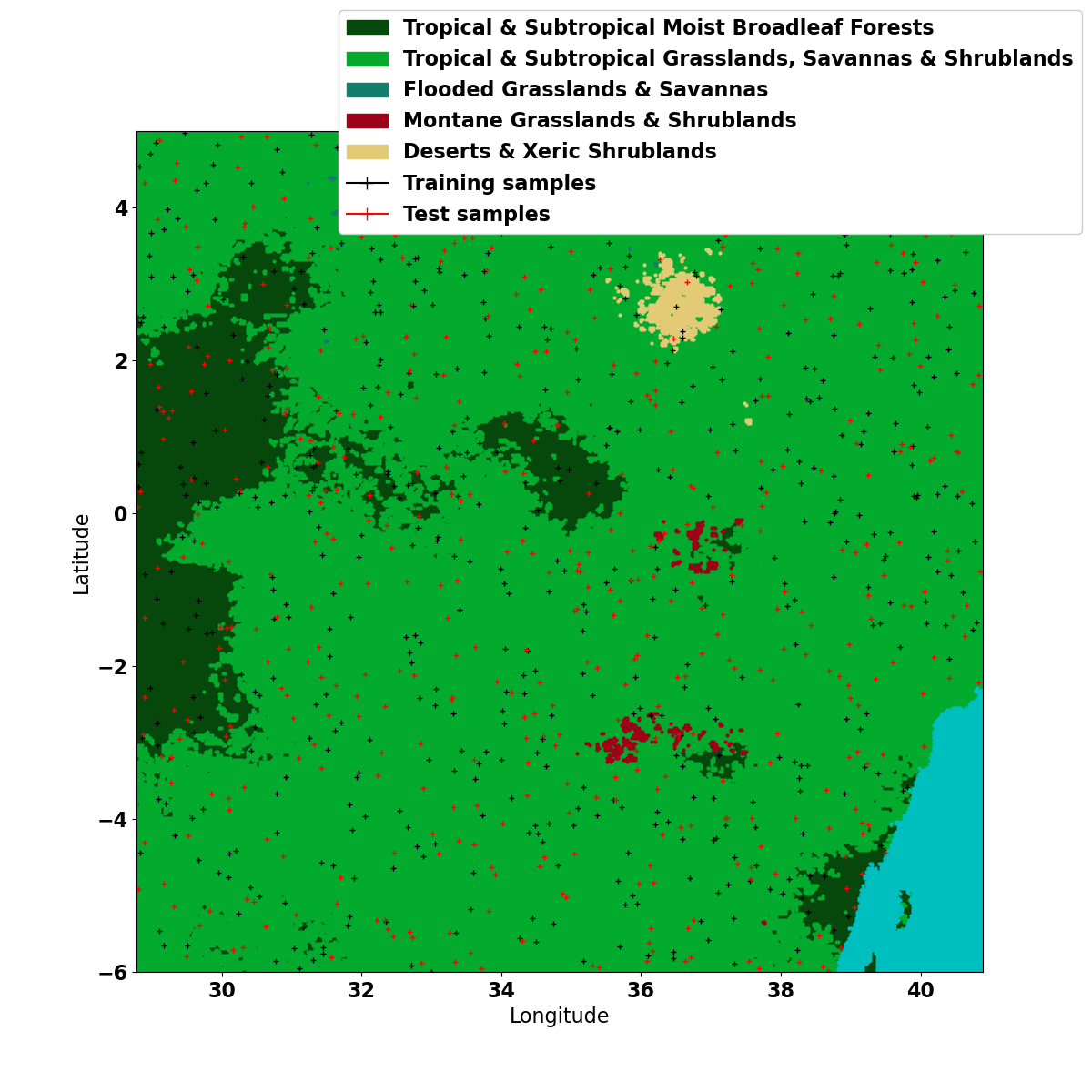}
    \vspace{-0.8cm}
    \caption{Taxabind}
    \label{fig:lvb_taxabind}
    \end{subfigure}
    \begin{subfigure}{.45\textwidth}
    \centering
    \includegraphics[width=1\linewidth]{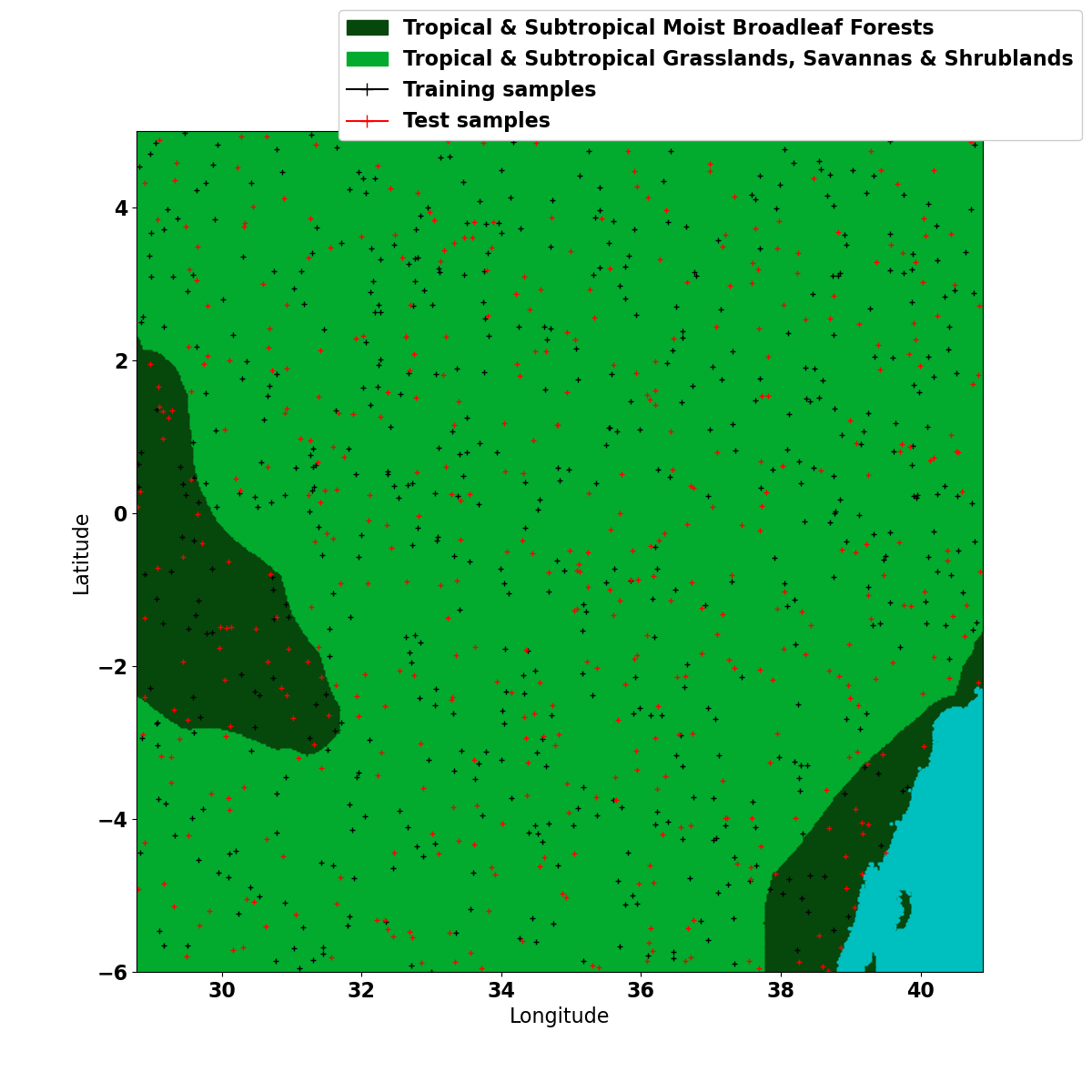}
    \vspace{-0.8cm}
    \caption{SINR}
    \label{fig:lvb_sinr}
    \end{subfigure}
    \newline
\caption{Comparison of the fully-trained and pretrained embeddings by mapping the predicted biomes in the region around Lake Victoria, covering parts of Kenya, Tanzania, Uganda, Rwanda, Burundi and the Democratic Republic Kongo. \method{Climplicit} achieves a good reconstruction without any hallucinations, albeit missing high-level details such as the mangroves along the coast.}
\label{fig:lvb_comp}
\end{figure}

\begin{figure}[!h]
    \centering
    \begin{subfigure}{.43\textwidth}
    \centering
    \includegraphics[width=1\linewidth]{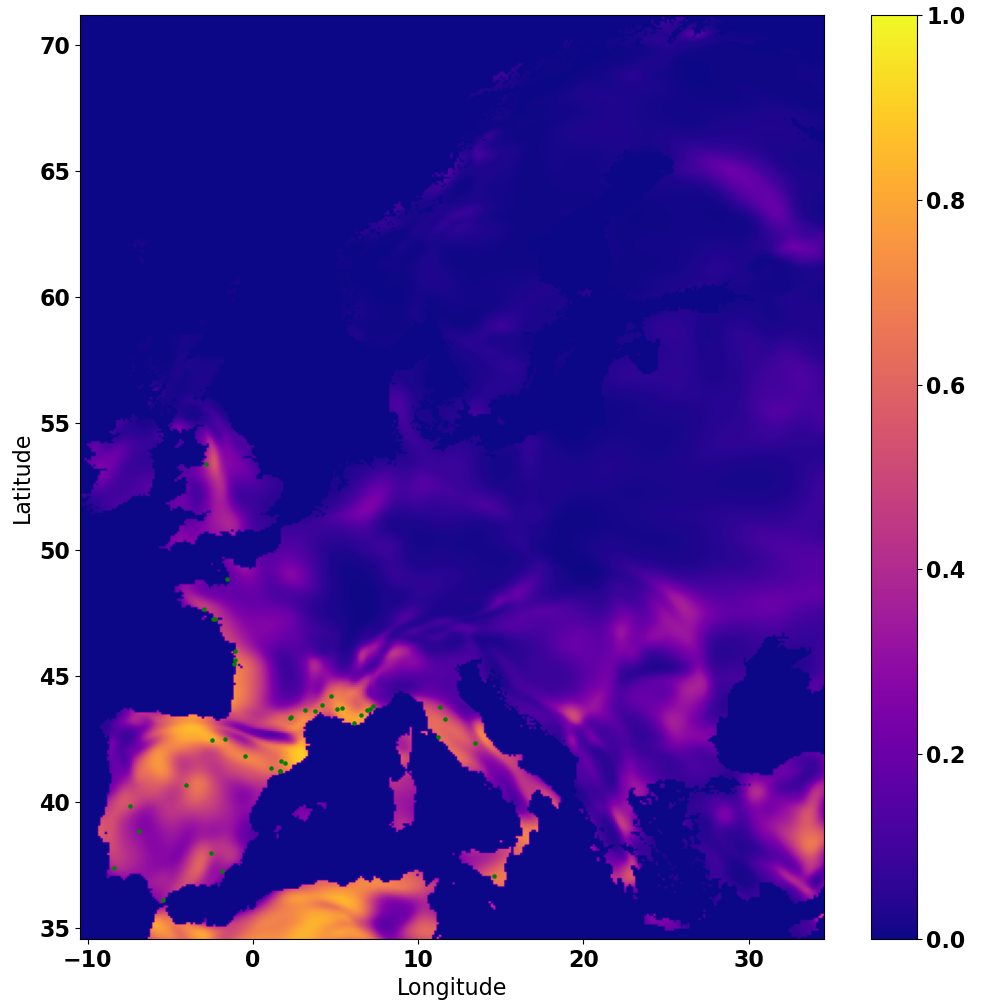}
    \caption{March}
    \label{fig:month_comparison_Climplicit}
    \end{subfigure}
    \begin{subfigure}{.43\textwidth}
    \centering
    \includegraphics[width=1\linewidth]{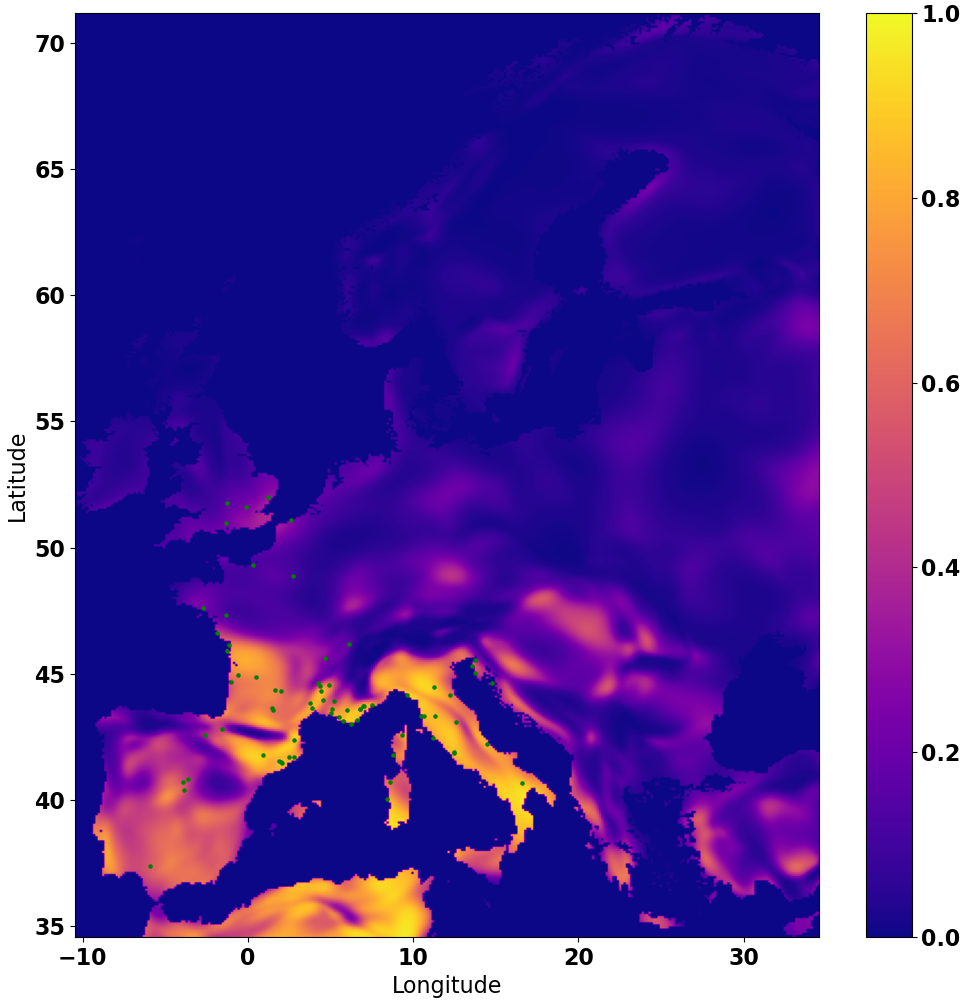}
    \caption{June}
    \label{fig:month_comparison_satclip}
    \end{subfigure}

    \begin{subfigure}{.43\textwidth}
    \centering
    \includegraphics[width=1\linewidth]{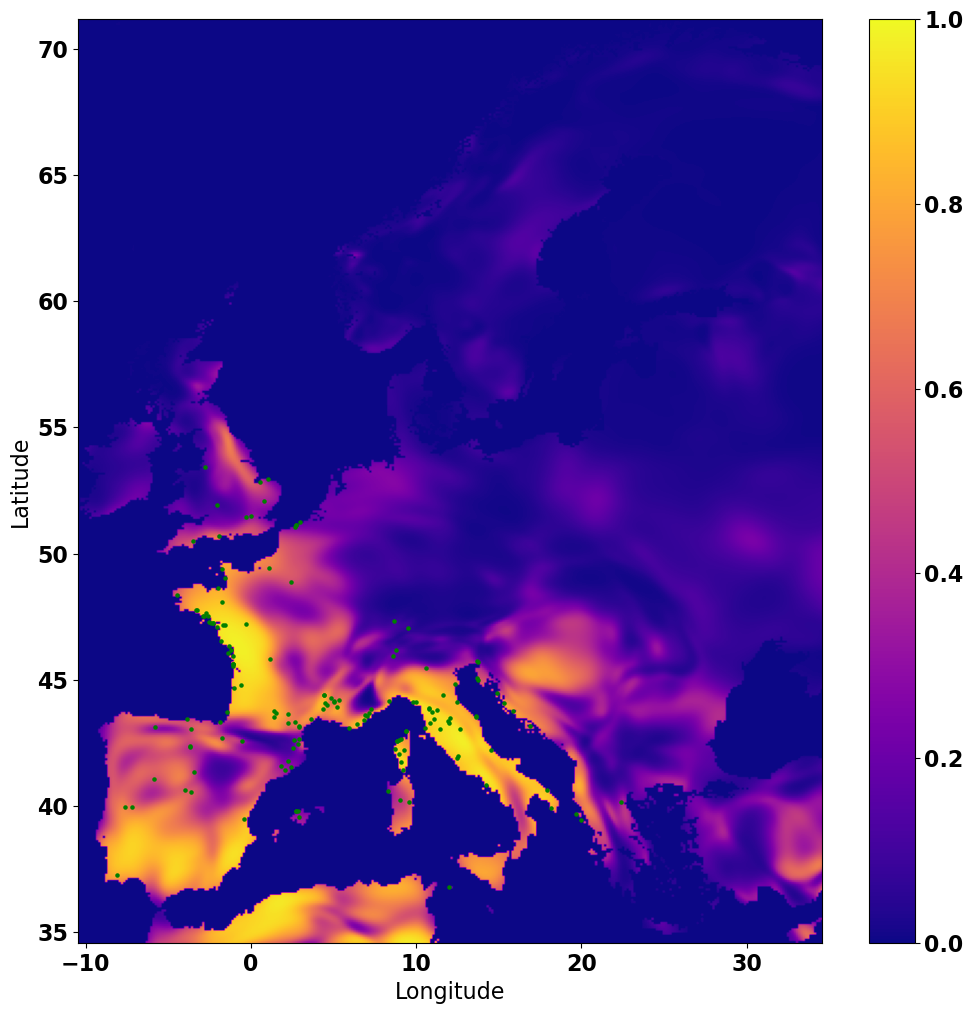}
    \caption{September}
    \label{fig:month_comparison_sdm_true_16}
    \end{subfigure}
    \begin{subfigure}{.43\textwidth}
    \centering
    \includegraphics[width=1\linewidth]{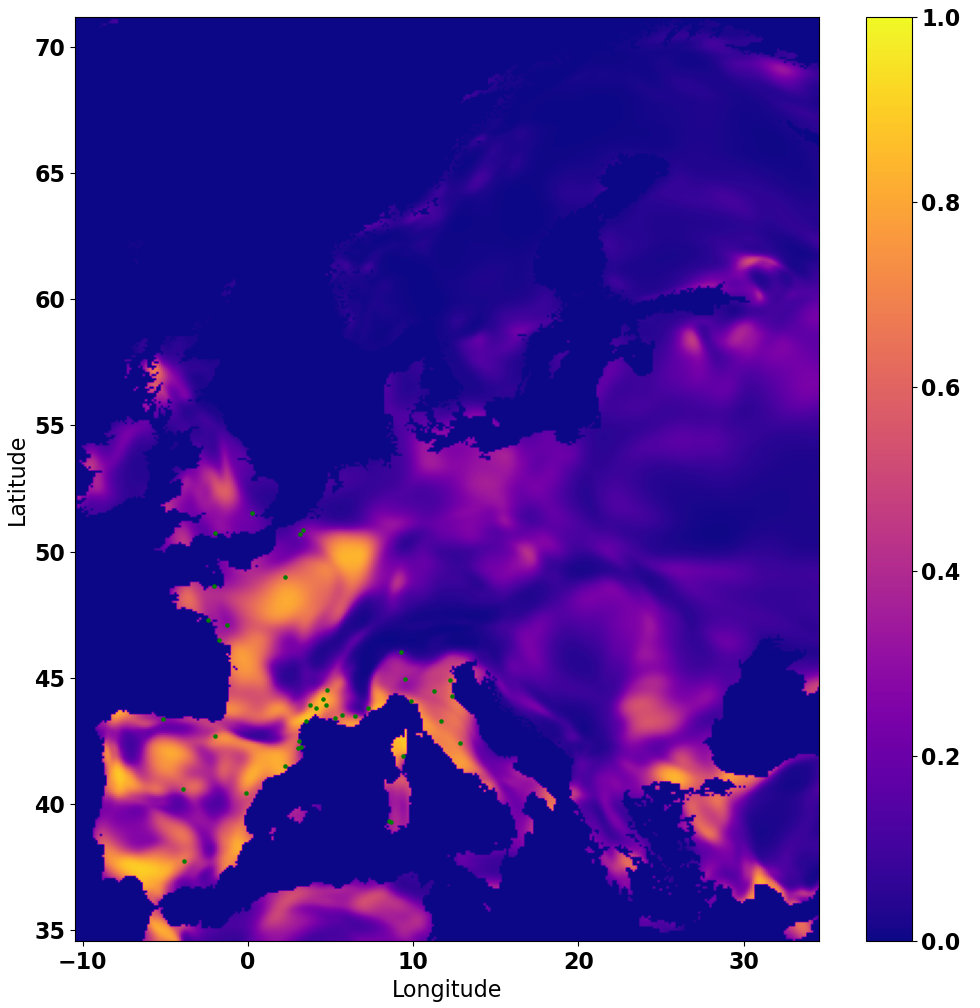}
    \caption{December}
    \label{fig:month_comparison_sdm_dec_T16}
    \end{subfigure}

    \begin{subfigure}{.43\textwidth}
    \centering
    \includegraphics[width=1\linewidth]{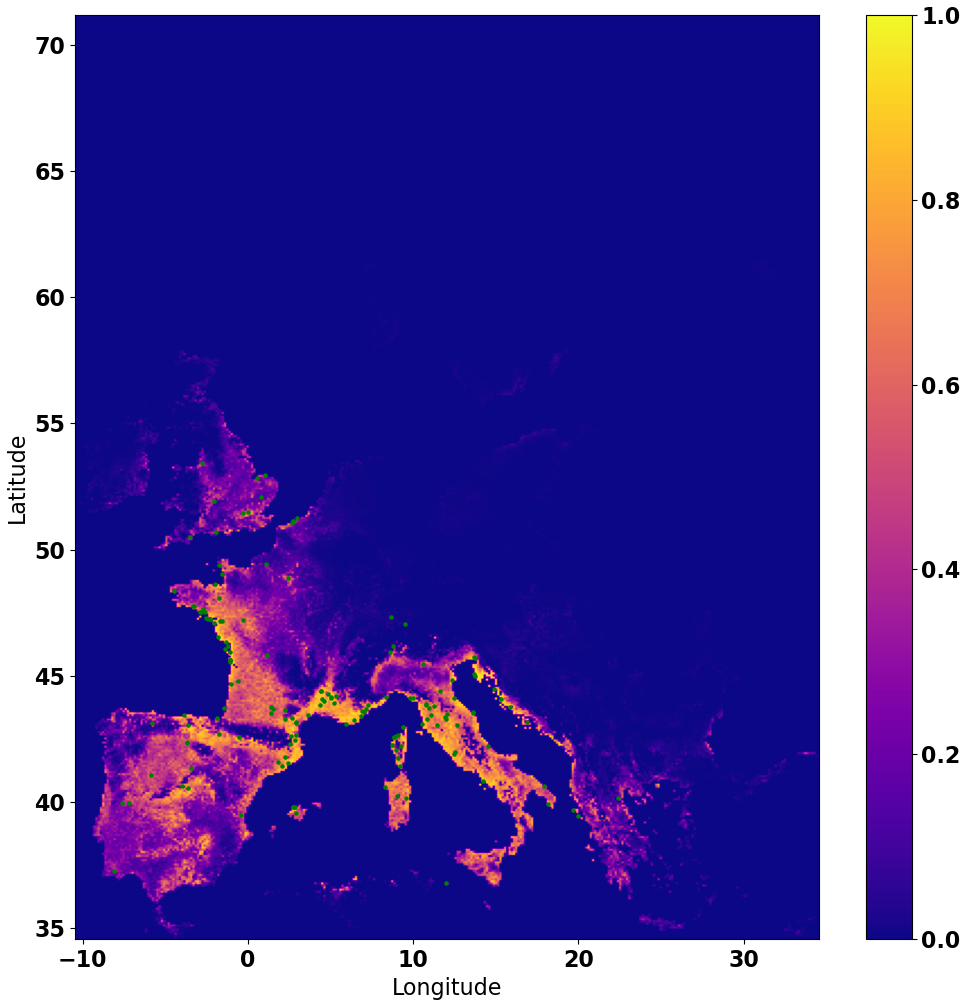}
    \caption{FS LOC $+$ CH}
    \label{fig:month_comparison_sdm_LOCCH}
    \end{subfigure}
    \begin{subfigure}{.43\textwidth}
    \centering
    \includegraphics[width=1\linewidth]{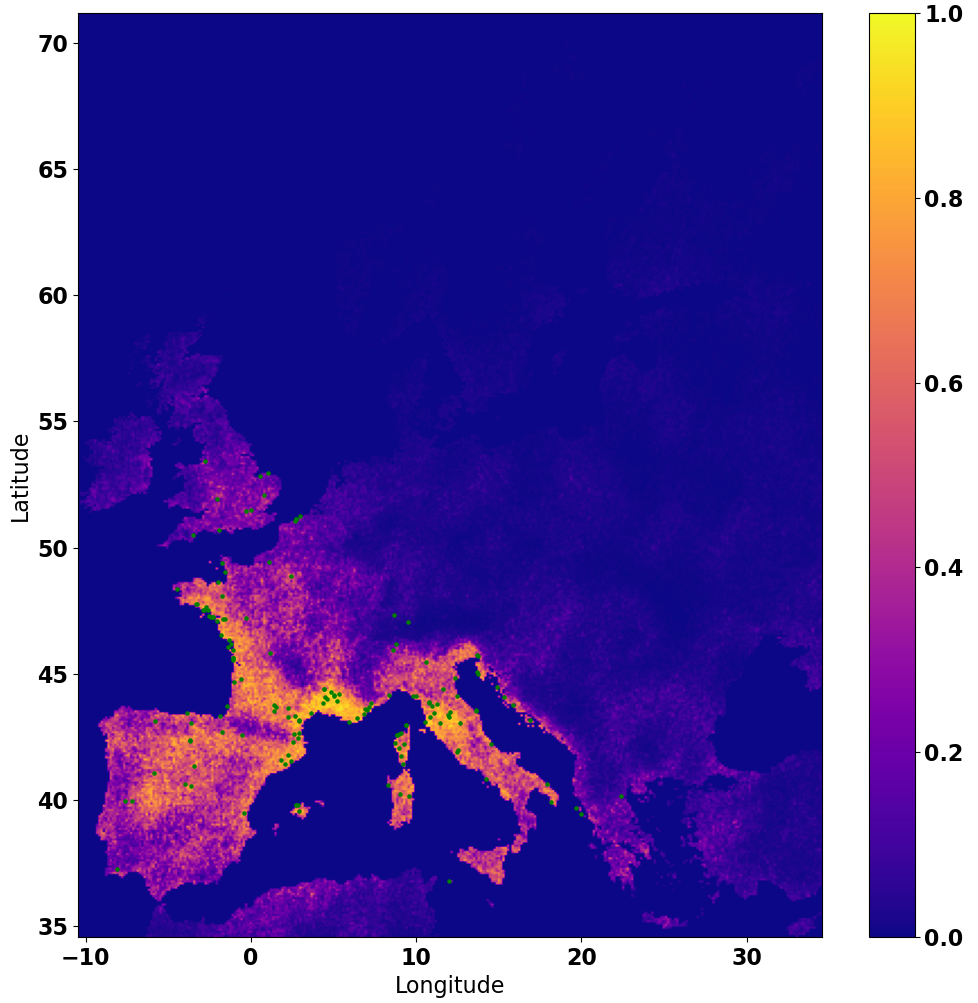}
    \caption{GeoCLIP}
    \label{fig:month_comparison_sdm_geoclip}
    \end{subfigure}
\caption{Visualization of the species distribution for \textit{Quercus ilex} (Stone oak) for four months learned by probing the \method{Climplicit} embeddings in the SDM task. For each location, the model provides a probability of observing the oak (Blue: low, Yellow: high). The green dots represent the citizen scientist records of Quercus ilex for the respective month. Due to using months as input, the model learns different embeddings for each month, which allows for diverse distribution learning across the year. Additionally we present the distributions for the two strongest baselines.}
\label{fig:month_comparison}
\end{figure}

\begin{figure}[!h]
    \centering
    \begin{subfigure}{.43\textwidth}
    \centering
    \includegraphics[width=1\linewidth]{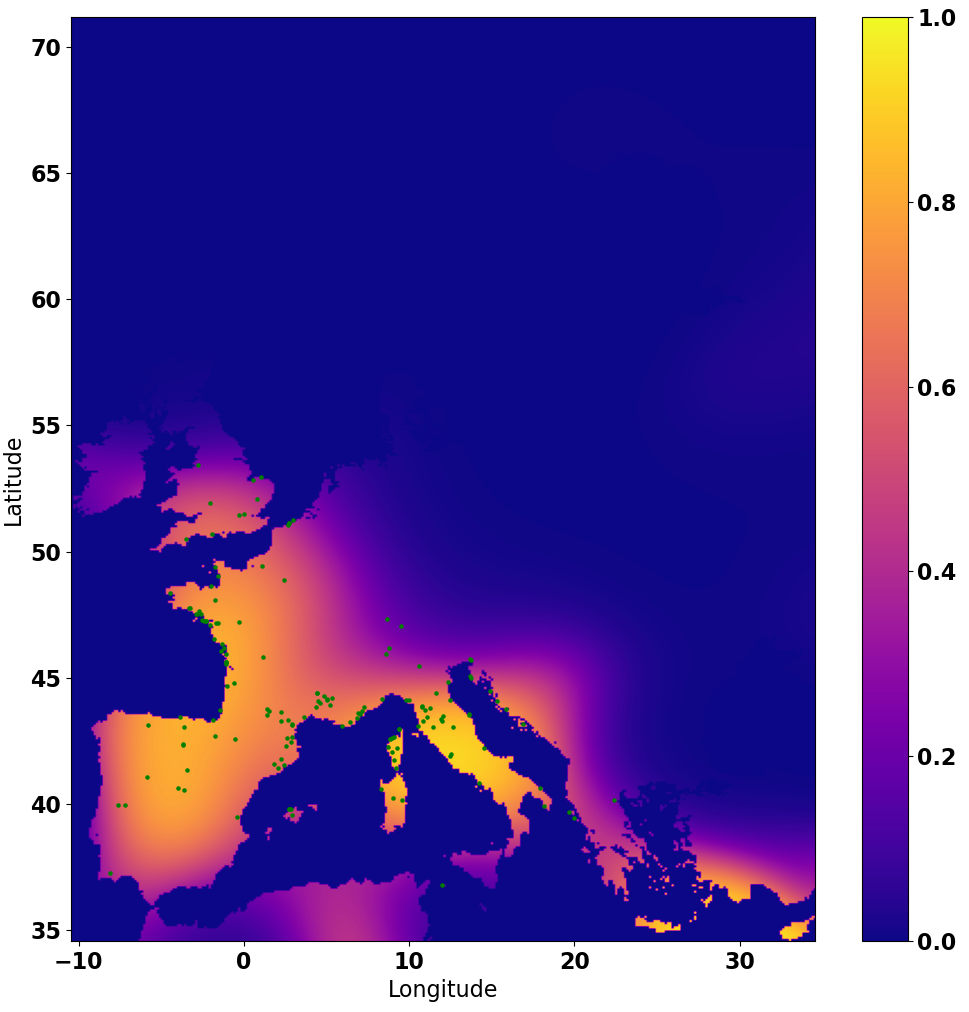}
    \caption{2 layers}
    \label{fig:depth_comparison_Climplicit}
    \end{subfigure}
    \begin{subfigure}{.43\textwidth}
    \centering
    \includegraphics[width=1\linewidth]{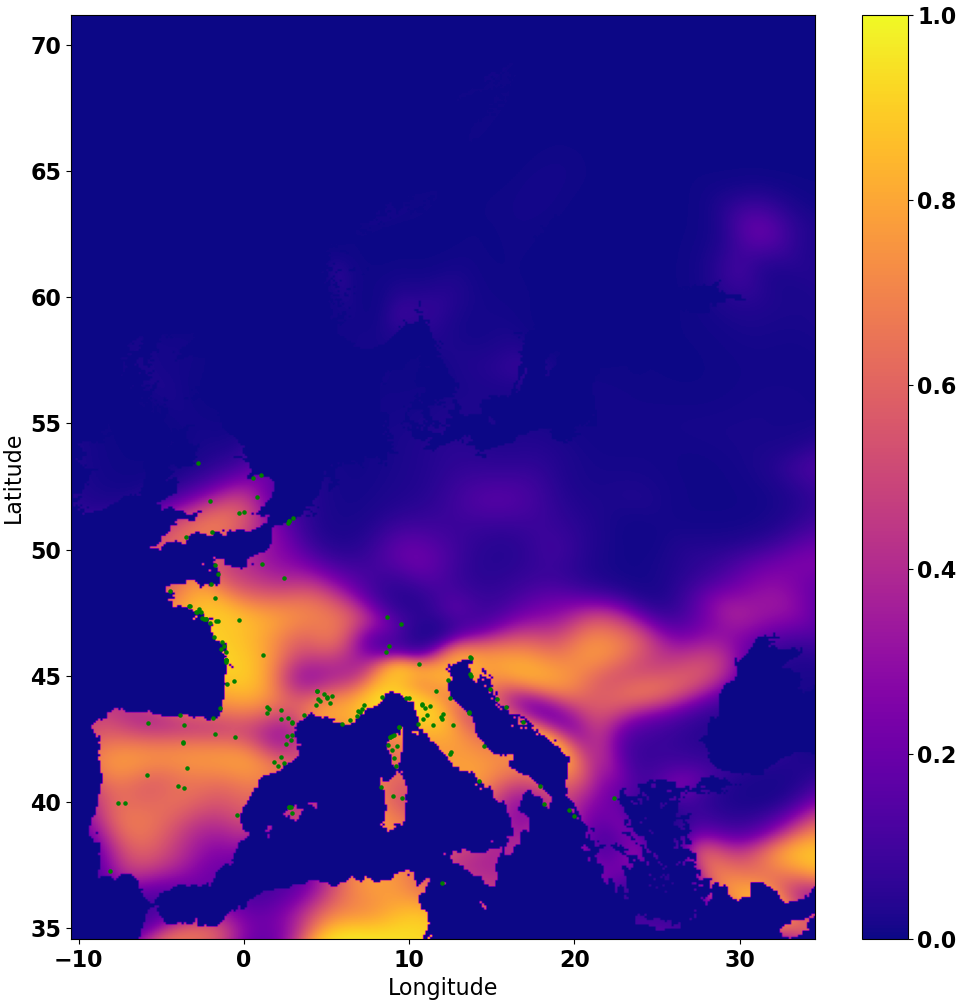}
    \caption{4 layers}
    \label{fig:depth_comparison_satclip}
    \end{subfigure}

    \begin{subfigure}{.43\textwidth}
    \centering
    \includegraphics[width=1\linewidth]{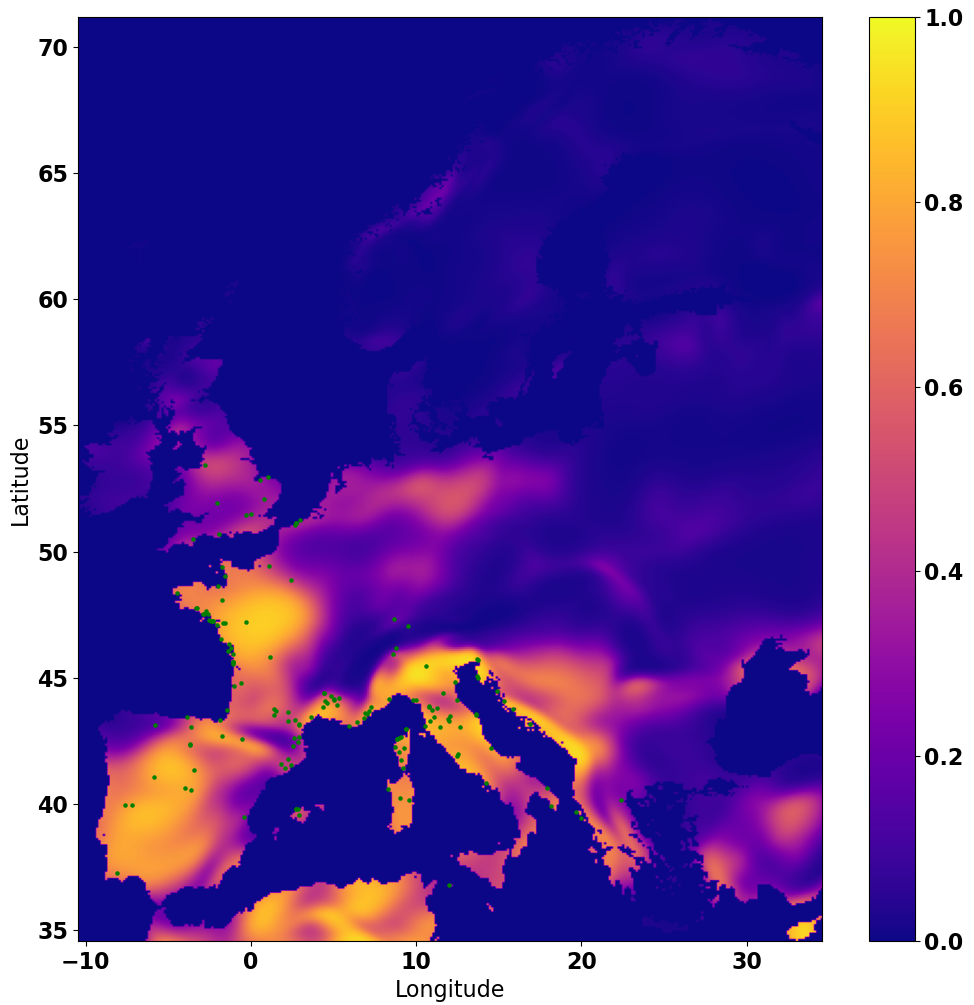}
    \caption{8 layers}
    \label{fig:depth_comparison_sdm_true_8}
    \end{subfigure}
    \begin{subfigure}{.43\textwidth}
    \centering
    \includegraphics[width=1\linewidth]{figures/images/sdm/sdm_true_16.png}
    \caption{16 layers}
    \label{fig:depth_comparison_sdm_true_16}
    \end{subfigure}

    \begin{subfigure}{.43\textwidth}
    \centering
    \includegraphics[width=1\linewidth]{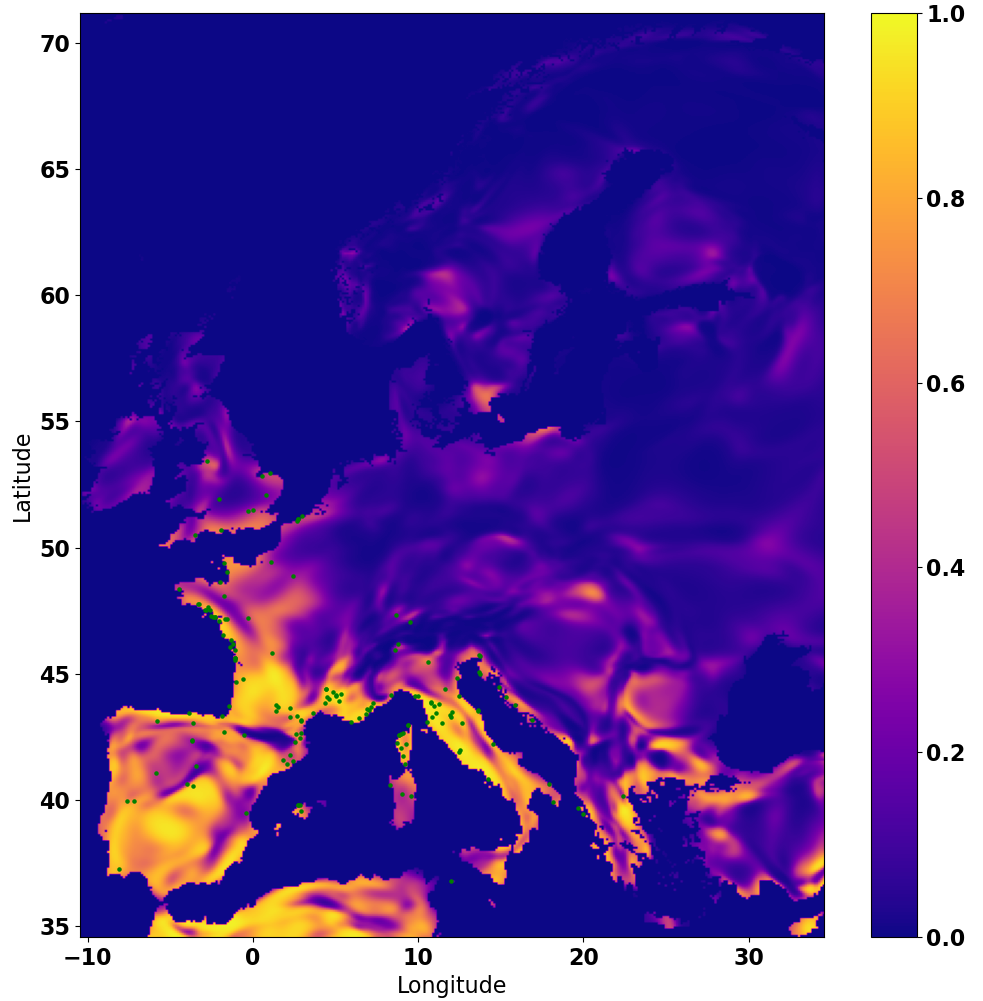}
    \caption{32 layers}
    \label{fig:depth_comparison_sdm_true_32}
    \end{subfigure}
    \begin{subfigure}{.43\textwidth}
    \centering
    \includegraphics[width=1\linewidth]{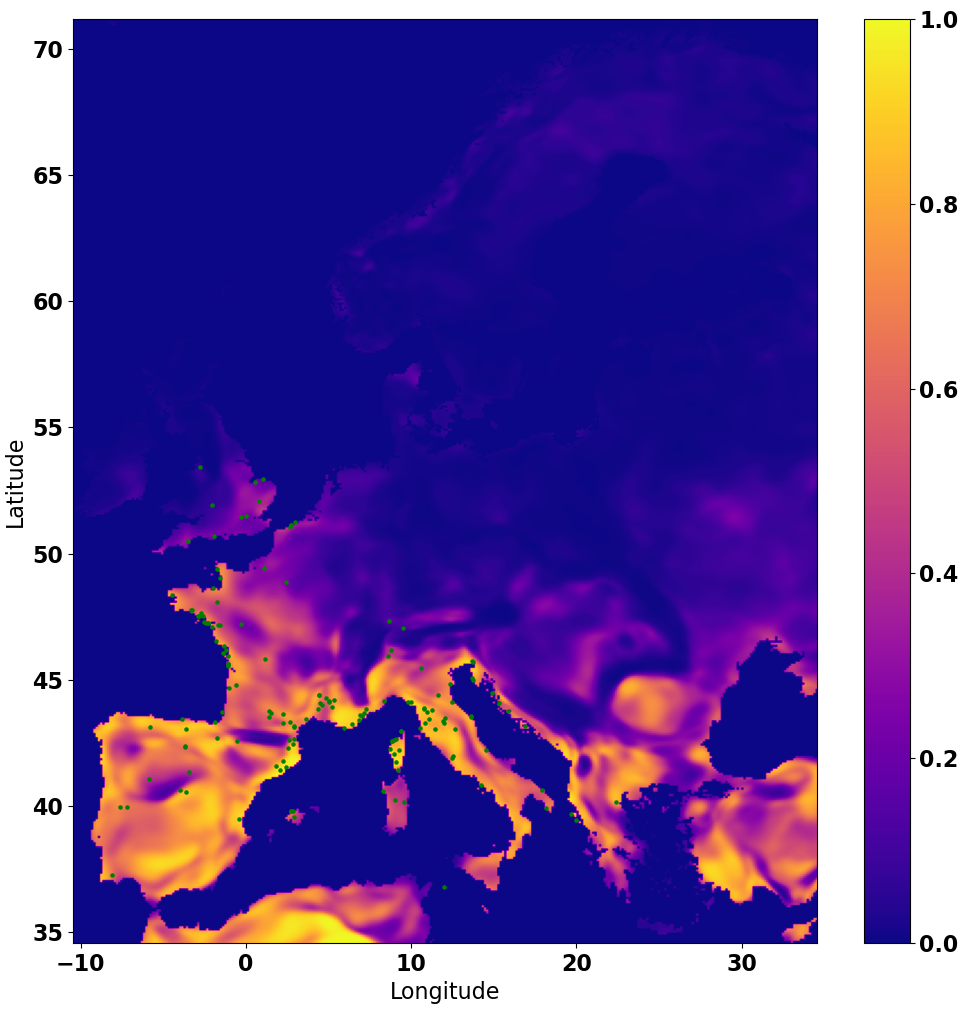}
    \caption{SIREN 32 layers}
    \label{fig:depth_comparison_sdm_false_32}
    \end{subfigure}
\caption{Visualization of the species distribution for \textit{Quercus ilex} (Stone oak) after probing the \method{ReSIREN} networks pretrained with different numbers of layers. All distributions are for the month of September. Blue to yellow signifies the predicted probability and green signifies the occurrences of the oak in September. It shows that with increasing depth, the level of detail also increases, though at the cost of increasing extrapolation (e.g. Skåne county in Sweden). When using MLP probing, we do not observe such extrapolation anymore. We also present the distribution learned by a 32 layer \method{SIREN}.}
\label{fig:depth_comparison}
\end{figure}

\end{document}